\documentclass{article} 
\usepackage{iclr2026_conference,times}

\usepackage{amsmath,amsfonts,bm}

\def\eqref#1{equation~\ref{#1}}

\def\1{\bm{1}}

\usepackage{amssymb}
\usepackage{amsthm}
\theoremstyle{plain} 
\newtheorem{theorem}{Theorem}[section]      
\newtheorem{lemma}[theorem]{Lemma}          
\newtheorem{corollary}[theorem]{Corollary}  
\newtheorem{proposition}[theorem]{Proposition}
\usepackage{graphicx}
\usepackage{subcaption} 
\theoremstyle{definition} 
\newtheorem{definition}[theorem]{Definition}
\newtheorem{assumption}[theorem]{Assumption}

\theoremstyle{remark} 
\newtheorem*{remark}{Remark}

\DeclareMathAlphabet{\mathsfit}{\encodingdefault}{\sfdefault}{m}{sl}
\SetMathAlphabet{\mathsfit}{bold}{\encodingdefault}{\sfdefault}{bx}{n}

 \usepackage{multirow}
\usepackage{tabularx,makecell,array,xcolor,booktabs}
\definecolor{emerald}{RGB}{0,158,115}
\definecolor{violet}{RGB}{148,0,211}
\newcolumntype{Y}{>{\raggedleft\arraybackslash}X}

\usepackage{xcolor}

\usepackage{mathrsfs}

\usepackage{hyperref}
\usepackage{adjustbox}
\usepackage{url}
\usepackage{enumitem}

\title{Adaptive Mamba Neural Operators}

\author{Zeyuan Song \& Zheyu Jiang\\
School of Chemical Engineering\\
Oklahoma State University\\
Stillwater, OK 74078, USA \\
\texttt{\{taekwon.song,zheyu.jiang\}@okstate.edu}
}

\iclrfinalcopy 
\begin{document}

\maketitle

\begin{abstract}
Accurately solving partial differential equations (PDEs) on arbitrary geometries and a variety of meshes is an important task in science and engineering applications. In this paper, we propose Adaptive Mamba Neural Operators (AMO), which integrates reproducing kernels for state-space models (SSMs) rather than the kernel integral formulation of SSMs. This is achieved by constructing Takenaka-Malmquist systems for the PDEs. AMO offers new representations that align well with the adaptive Fourier decomposition (AFD) theory and can approximate the solution manifold of PDEs on a wide range of geometries and meshes. In several challenging benchmark PDE problems in the fields of fluid physics, solid physics, and finance on point clouds, structured meshes, regular grids, and irregular domains, AMO consistently outperforms state-of-the-art solvers in terms of relative $L^2$ error. Overall, this work presents a new paradigm for designing explainable neural operator frameworks. The code is available at \url{https://github.com/checlams/AMO}.
\end{abstract}

\section{Introduction}
A wide range of scientific and engineering phenomena, including fluid dynamics, heat and mass transport, structural mechanics, and cell growth, can be characterized and modeled by partial differential equations (PDEs). Most nonlinear PDEs do not have analytical solutions and need to be solved numerically. Traditional discretization-based approaches for solving PDEs can be computationally expensive. To speed up the solution process, neural operators have recently been proposed as an extension of neural networks to learn the infinite-dimensional solution operators of various PDE problems. It has been proven that, with finite-dimensional solutions as training data, neural operators can accurately learn the infinite-dimensional solution space. Once learned, neural operators are mesh-independent, so neural operators trained on coarse grids can generalize to finer grids. 

Frequency-based neural operators, such as Fourier neural operator (FNO) \citep{li2020fourier}, wavelet neural operator (WNO) \citep{tripura2023wavelet}, multiwavelet transform (MWT) \citep{gupta2021multiwavelet}, U-shaped neural operator \citep{rahman2022u}, spectral neural operator \citep{fanaskov2023spectral}, and latent spectral model (LSM) \citep{wu2023solving}, are attractive since the solution space of many PDEs can be naturally expressed in spectral bases. Frequency-based neural operators approximate the PDE solutions by learning how frequencies evolve, and nonlinear terms become convolution in the associated frequency domain. However, the performance of existing frequency-based neural operators may deteriorate in irregular geometries \citep{li2023fourier}, as their associated bases could lose orthogonality and eigenfunction properties in irregular domains \citep{lingsch2023beyond,chen2024learning}. As a result, retaining these important properties for kernels and bases for irregular domains is critical.

Along this line, a recently proposed neural operator solver, latent Mamba operator (LaMO) \citep{tiwari2025latent}, shows great promise in capturing PDE solutions on irregular domains. LaMO integrates the efficiency of state-space models (SSMs) in latent space with the expressive power of kernel integral formulations in neural operators. Although the selective convolution kernels utilized in LaMO can effectively capture PDE solutions on irregular domains, their lack of orthogonality property may lead to spectral mixing. Furthermore, the kernels in LaMO are finite-order linear dynamic filters \citep{gu2023mamba}, which may introduce a low-pass filtering bias, leading to poor recovery of high-frequency and singular features \citep{gu2021efficiently,gu2023mamba}. In the illustrative experiments discussed in Appendix \ref{appen_mex}, we show that LaMO suffers from deviation in the propagation of high-frequency perturbations for 1-D advection PDE, and it fails to capture the singularities in 2-D Darcy flow equation with fractal noise as the permeability field.

Recognizing the fact that LaMO lacks frequency-domain implementation, here we propose Adaptive Mamba Neural Operator (AMO), a novel neural operator architecture that synergizes an adaptive Fourier decomposition with the efficiency of structured SSMs in the frenquency domain \citep{gu2023mamba,parnichkun2024state}. AMO parameterizes the SSM transfer function in a Takenaka-Malmquist (TM) system in a reproducing kernel Hilbert space (RKHS), thus allowing state-free kernel construction and inference directly on the spectrum. The Mamba blocks in AMO serve as rational filters while retaining linear-time selective scanning. Furthermore, it turns out that AMO structure resembles adaptive Fourier decomposition (AFD), a novel signal decomposition technique achieving higher accuracy and significant computational speedup compared to conventional signal decomposition methods \citep{qian2010intrinsic,qian2012adaptive}. The architecture and design of AMO is fully guided by the AFD theory, thereby improving the mathematical explainability and groundness of AMO. 

Overall, our key contributions are summarized as follows:
\begin{enumerate}[leftmargin=*]
    \item AMO is the first neural operator which explicitly incorporates TM systems and Fourier-based methods into the Mamba structure. AMO accurately solves the PDE problems on diverse geometries and effectively handles singularities and long-range dependencies of PDE solutions, Furthermore, we develop theoretical foundations for AMO and prove that AMO performs AFD approximation of PDE solutions. 
    \item The design of every component of AMO is fully guided by the AFD theory, leading to a mathematically interpretable and grounded architecture. Using a TM layer, AMO projects the input into TM systems in a Hardy space, and constructs the reproducing kernels from adaptively selected poles. These adaptive poles serve to construct the reproducing kernels adaptively. We demonstrate the importance of utilizing adaptive poles as opposed to fixed poles and investigate how the number of adaptive poles influences the performance of AMO.
    \item AMO outperforms state-of-the-art neural operator solvers in terms of accuracy across a diverse set of benchmark PDE problems, including plasticity, elasticity, airfoil, pipe flow, Navier-Stokes, and Darcy flow on various geometries. It also achieves outstanding performance in financial applications, such as solving the Black-Scholes equation for the European option pricing problem.
\end{enumerate}

\section{Related work}
\label{rel_w}
\paragraph{Frequency-based neural operators.} Early advancements in operator learning exploited spectral decompositions to encode global information efficiently. A notable example is FNO \citep{li2020fourier}, which parameterizes integral kernels in the Fourier domain to enable resolution‐invariance. However, FNO does not generalize well to irregular geometries \citep{li2020fourier}. Later, Geo-FNO \citep{li2023fourier} was proposed to solve PDEs on general geometries. U-FNO \citep{wen2022u} introduced architectural modifications to better capture localized details while maintaining FNO's global properties. Meanwhile, F-FNO \citep{tran2021factorized} generalizes the FNO architecture for more efficient spectral layers and deeper architectures. On the other hand, neural operators based on the wavelet transform include WNO \citep{tripura2023wavelet}, MWT \citep{gupta2021multiwavelet}, Pad{\'e} \citep{gupta2022non}, and CMWNO \citep{xiao2023coupled}. Fourier and wavelet transforms are both special cases of spectral decomposition, and neural operators based on spectral decomposition has recently been proposed  \citep{fanaskov2023spectral}.

\paragraph{Attention-based neural operators.} Attention mechanisms have been widely studied in neural operator domain. Some of the notable works include orthogonal attention \citep{xiao2023improved}, physics-cross-attention \citep{wang2024latent}, and nonlocal attention \citep{yu2024nonlocal}. The Transformer structure is also a promising building block for neural operators. Some of the related works include OFormer \citep{li2022transformer}, LSM \citep{wu2023solving}, and Transolver \citep{wu2024transolver}. However, Transformers struggle to capture kernel integral transforms efficiently in complex, high-dimensional continuous PDEs \citep{guibas2021adaptive}. 

\paragraph{SSM-based neural operators.} To address the computational inefficiency of Transformer-based neural operators, SSM and Mamba emerge as promising architectures for neural operator designs \citep{tiwari2025latent}. Previous studies of SSM-based neural operators \citep{zheng2024alias,cheng2024mamba,hu2024state,tiwari2025latent} have been applied to nonlinear PDEs on irregular geometries and dynamical systems. These works incorporate traditional SSMs with different scan strategies without considering the information in the frequency domain. On the other hand, our AMO considers the frequency information via its explicit kernel and SSMs from a transfer function perspective \citep{parnichkun2024state}.

\section{Adaptive Fourier Mamba operator}
\label{AMO}
\subsection{Problem Statement}
We frame our task as learning a solution operator for a family of parametric PDEs. In general, consider a PDE defined on a spatial domain $\Omega \subset \mathbb{R}^d$ and a time interval $(0, T]$:
\begin{equation}\label{eqn_problem}
    \mathcal{L}_a [u(x,t)] = f(x,t), \quad \forall (x,t) \in D \times (0,T],
\end{equation}
which is subject to a set of initial and boundary conditions. Here, the parameter function $a\in \mathcal{A}$ specifies the coefficients and initial and boundary conditions of Equation \ref{eqn_problem}. In operator learning, our goal is to construct an accurate approximation for $\mathcal{G}: \mathcal{A} \rightarrow \mathcal{F}(D \times [0,T])$, which maps the parameter function $a$ to the corresponding solution function $u(x,t)\in \mathcal{F}$, via a parametric mapping $\mathcal{G}_\theta$. The aim is to learn $\theta$ such that $\mathcal{G}_\theta \approx \mathcal{G}$ from a set of training data $\{(a_j,u_j)\}_{j}$.

\subsection{AMO Architecture}
AMO is a novel neural operator architecture that synergizes the mathematical groundness of AFD theory with the efficiency of structured SSMs in the frenquency domain \citep{gu2023mamba,parnichkun2024state}. Different from LaMO \citep{tiwari2025latent}, which compresses the physical tokens into a fixed-size latent representation, AMO utilizes a multi-layer fully-connected feedforward neural network (MLP) to first map the encoded tokens to their counterparts on the reproducing kernel Hilbert space (RKHS), and then iteratively refine them by a series of processing blocks. Each block uniquely integrates two components: (i) a TM layer containing global spectral transform via data-dependent TM bases, and (ii) a bidirectional SSM \citep{gu2021efficiently,gu2023mamba} parameterized by transfer functions in the frequency domain \citep{parnichkun2024state} to efficiently capture long-range dependencies within the RKHS.

\paragraph{Neural architecture.} Given the parameter function (input) $a$, the output of AMO, denoted as $\hat{u}_{N,\theta}$, is:
\begin{equation}
   \hat{u}_{N,\theta}= \mathcal{G}_\theta(a) =\left( \mathcal{Q} \circ \mathcal{S}^N \circ \mathcal{L}^N \circ \dots \circ \mathcal{S}^1 \circ \mathcal{L}^1 \circ \mathcal{R}\circ \mathcal{P}\right)(a),
\end{equation}
where $\circ$ is the function composition, $N$ is the number of processing blocks, $\mathcal{P}$ is the lifting operator which encodes into a lower-dimensional space (maps the input to the first latent representation $\mathbf{z}_0$) \citep{tiwari2025latent,li2020fourier}, $\mathcal{Q}$ is the corresponding projection operator mapping the lower-dimensional space back to the original space (maps the final latent representation $\mathbf{z}_{N+1}$ to the output) \citep{tiwari2025latent,li2020fourier}, $\mathcal{R}$ is a multi-layer neural network mapping the physical token to an RKHS, $\mathcal{L}^{i}=\mathrm{SSM}^i\circ \mathrm{TM}^i$ ($i=1,\ldots,N$) is the processing block of AMO (which consists of a TM layer and a bidirectional SSM), and $\mathcal{S}^i$ ($i=1,\ldots,N$) are aggregation layers with skip connections. These aggregation layers not only receive the final output from the layer sequence but also have access to the intermediate outputs from each of the preceding layers.

\paragraph{The lifting operator,} $\mathcal{P}$, projects the $N_s$ physical token inputs into a compressed set of $M$ encoded tokens, where $M \ll N_s$. This projection is achieved via a cross-attention mechanism. A learnable query array, $\mathbf{L} \in \mathbb{R}^{M \times D_{\text{embed}}}$, acts as the query. The key and value pairs are constructed by combining a linear projection of the input features $\mathbf{x}_{\text{phys}}$ with a positional embedding of their coordinates $\mathbf{g}_{\text{phys}}$ generated by a positional encoding network $\mathrm{PEN}$. Here, $\mathbf{x}_{\text{phys}}\in\mathbb{R}^{N_s\times D_{\text{in}}}$ stacks the feature vectors $\{\mathbf{x}_i\}_{i=1}^{N_s}$ and $\mathbf{g}_{\text{phys}}\in\mathbb{R}^{N_s\times d}$ stacks the coordinates $\{\mathbf{g}_i\}_{i=1}^{N_s}$, and the physical token is essentially pair $(\mathbf{g}_i,\mathbf{x}_i)$.  The process for generating the initial representation $\mathbf{z}_0$ is formally defined as:
\begin{equation}
\begin{aligned}
\mathbf{kv} &= \mathrm{Linear}(\mathbf{x}_{\text{phys}}) + \mathrm{PEN}(\mathbf{g}_{\text{phys}}), \\
\mathbf{z}'_0 &= \mathrm{CrossAttn}(\text{query}=\mathbf{L}, \text{key}=\mathbf{kv}, \text{value}=\mathbf{kv}), \\
\mathbf{z}_0 &= \mathbf{z}'_0 + \mathrm{FFN}(\mathbf{z}'_0),
\end{aligned}
\end{equation}
where the output of the cross-attention module is processed through a residual connection and a standard feed-forward network $\mathrm{FFN}$.

\paragraph{The mapping operator,} denoted by $\mathcal{R}$, acts on the encoded representation produced by the lifting operator $\mathcal{P}$, which transforms this discrete encoded tokens into a representation within a continuous function space. Let $\mathbf{z}_0 \in \mathbb{R}^{M \times D_{\text{embed}}}$ be the set of encoded tokens generated by $\mathcal{P}$, the operator $\mathcal{R}: \mathbb{R}^{M \times D_{\text{embed}}} \to \mathcal{H}$ maps this representation to its counterpart in an RKHS $\mathcal{H}$. This mapping is typically implemented as a multi-layer fully-connected feedforward network $\mathrm{MLP}$, which processes each token independently as:
\begin{equation}
\mathbf{z}_1 = \mathcal{R}(\mathbf{z}_0) = \mathrm{MLP}(\mathbf{z}_0),  
\end{equation}
where $\mathbf{z}_1$ denotes the projected tokens in the RKHS. We remark that, the mapping operator $\mathcal{R}$ maps the encoded tokens $\mathbf{z}_0$ to the new tokens $\mathbf{z}_1$ in $\mathcal{H}$ without knowing the physical information $\mathbf{x}_{\text{phys}}$ and $\mathbf{g}_{\text{phys}}$.

\paragraph{The TM layer,} denoted by $\mathrm{TM}^i$ ($i=1,\ldots,N$), performs a global convolution via a spectral transform, where the reproducing kernels and TM bases are constructed from data-dependent poles. To define the reproducing kernels, we parameterize a small $\mathrm{MLP}$ to predict a set of $i$ complex values called ``poles'' $\{a_k\}_{k=1}^i$ (denoted as $a_{1:i}$) located in the unit disk $\mathbb{D} = \{z\in\mathbb{C}: |z|<1\}$ from tokens $\mathbf{z}_{i}$. Once we have the set of poles, we can explicitly define the reproducing kernel $K_a(z)$ as:
\begin{equation}\label{eqn_ker}
    K_a(z) = \frac{1}{1-\overline{a}z},
\end{equation}
where $z\in \mathcal{H}$ and $a$ is a single pole satisfying $|a|<1$. Intuitively, we remark that each pole can be viewed as a ``tuning knob'' that selects a particular spatial pattern in the solution, with its location in the complex plane controlling how localized that pattern is. Adaptive poles allow AMO to survey more heavily in regions where the parameters change rapidly, while using fewer poles in smooth regions. Across layers, the poles evolve from broad, coarse patterns in early layers to more refined, problem-specific patterns in deeper layers.

To generalize on irregular geometries, the kernels in Equation \ref{eqn_ker} need to be modified to become orthonormal. These modified kernels are also known as the TM bases due to their deep connection to TM systems. The first basis, denoted as $\mathscr{B}_1$, is simply the normalized kernel of Equation \ref{eqn_ker} with pole $a_1$ as $\mathscr{B}_1(z;a_1) = \frac{\sqrt{1-|a_1|^2}}{1-\overline{a_1}z}$. Then, we start with $\frac{\sqrt{1-|a_2|^2}}{1-\overline{a_2}z}$, but it is not orthogonal to $\mathscr{B}_1$. We reach the orthogonality by subtracting its projection onto $\mathscr{B}_1$, and we get $\mathscr{B}_2(z;a_{1:2}) = \frac{\sqrt{1-|a_2|^2}}{1-\overline{a_2}z} \left( \frac{z-a_1}{1-\overline{a_1}z} \right)$ after normalization. This way, the bases $\mathscr{B}_i$ are finally formulated as:
\begin{equation}\label{eqn_tm}
\mathscr{B}_i(z;a_{1:i})=\frac{\sqrt{1-|a_i|^2}}{1-\overline{a_i}z}\prod_{j=1}^{i-1}\frac{z-a_j}{1-\overline{a_j}z},
\end{equation}
where $z\in\mathcal{H}$ and $a_{1:i}$ are poles learned by the small MLP satisfying $|a_k|<1$ for $k=1,\ldots,i$.
Overall, the $i$-th TM layer $\mathrm{TM}^i$ applies a small MLP $\mathbf{z_i}\mapsto a_{1:i}$, and then construct the TM bases $\mathscr{B}_i$ according to \ref{eqn_tm}. We remark that, the tokens $\mathbf{z}_i$ will be kept as the input of $\mathrm{SSM}^i$ along with the TM bases $\mathscr{B}_i$.

\paragraph{Bidirectional SSM block} is effective in solving PDEs on irregular geometries \citep{tiwari2025latent} and employs inherent kernel integrals. However, this inherent kernel does not contain information in the frequency domain, thereby falling short in capturing high-frequency and singular features. To address this limitation, we utilize the transfer function in training SSMs in the frequency domain \citep{parnichkun2024state}. The SSM block $\mathrm{SSM}^i$ generates the spectrum of output in the frequency domain $Y_i(e^{i\omega})$ as the product of the spectrum of input $Z(e^{i\omega})$ and the transfer function $H_i(e^{i\omega})$, i.e., $Z(e^{i\omega})H_i(e^{i\omega})$. We point out that the output is essentially the coefficient of discrete AFD operation with the form $\langle \mathbf{z_i},\mathscr{B}_i\rangle$ \citep{qian2010intrinsic,qian2011algorithm}, where the inner product is defined as $\langle x,f\rangle=\frac{1}{\Tilde{N}}\sum_{n=0}^{\Tilde{N}-1} x[n]\overline{f(e^{i2\pi n/\Tilde{N}})}$. Here, $\Tilde{N}$ denotes the length of signal $x=\{x[n]\}_{n=0}^{\Tilde{N}-1}$.

Let us consider the impulse response $h_i$ of SSM block $\mathrm{SSM}^i$ (linear time-invariant system) as: 
\begin{equation}
h_i[n]=\frac{1}{2\pi}\!\int_{0}^{2\pi}\overline{\mathscr{B}_i\left(e^{i\omega};a_{1:i}\right)}e^{i\omega n}\,d\omega.
\end{equation}

Then, the corresponding transfer function $H_i$ can be obtained as:
\begin{equation}\label{eqn_tran}
    H_i(e^{i\omega})=\overline{\mathscr{B}_i\left(e^{i\omega};a_{1:i}\right)}.
\end{equation}

By setting the transfer function of SSM to be Equation \ref{eqn_tran}, the SSM block computes a correlation of the input $\mathbf{z}_i$ and $\mathscr{B}_i$:
\begin{equation}\label{eqn_fre}
    Y_i(e^{i\omega})= H_i(e^{i\omega})X(e^{i\omega})=\overline{\mathscr{B}_i(e^{i\omega};a_{1:i})}\,X(e^{i\omega})
\end{equation}
in the frequency domain. In the time domain, Equation \ref{eqn_fre} leads to the update of $\mathbf{z_i}$:
\begin{equation}
\hat{\mathbf{z}}_\mathbf{{i+1}}[\ell]=(h_i * \mathbf{z_i})[\ell]=\sum_{n=0}^{M-1} \mathbf{z_i}[n]\overline{\mathscr{B}_i\big(e^{i2\pi(n-\ell)/M};a_{1:i}\big)},
\end{equation}
where $\ell$ denotes the time shift in the correlation operations. The zero-lag sample gives the final output:
\begin{equation}
\hat{\mathbf{z}}_\mathbf{{i+1}}[0]=(h_i * \mathbf{z_i})[0]=\sum_{n=0}^{M-1} \mathbf{z_i}[n]\overline{\mathscr{B}_i\big(e^{i2\pi n/M};a_{1:i}\big)}=\langle\mathbf{z_i},\mathscr{B}_i\rangle.
\end{equation}

\paragraph{Aggregation layers} $\mathcal{S}^i$ has $N$ neural layers and combines the skip connection $\mathbf{z_{i}}$ with the intermediate outputs $\hat{\mathbf{z}}_\mathbf{{i+1}}[0] = \mathcal{L}^i(\mathbf{z_{i}})$ and $\mathscr{B}_i = \mathrm{TM}^i(\mathbf{z_{i}})$:
\begin{equation}\label{eqn_aggre}
\begin{aligned}
    &\mathbf{z_{2}}=\mathcal{S}^i(\mathbf{z_{1}}, \hat{\mathbf{z}}_\mathbf{{2}}[0], \mathscr{B}_1) =  \hat{\mathbf{z}}_\mathbf{{2}}[0] \odot \mathscr{B}_1 \quad \text{for } i = 1,\\
    &\mathbf{z_{i+1}}=\mathcal{S}^i(\mathbf{z_{i}}, \hat{\mathbf{z}}_\mathbf{{i+1}}[0], \mathscr{B}_i) =\mathbf{z_{i}} + ( \hat{\mathbf{z}}_\mathbf{{i+1}}[0] \odot \mathscr{B}_i) \quad \text{for } i > 1,
    \end{aligned}
\end{equation}
where $\odot$ denotes the element-wise (Hadamard) product. 

\paragraph{Output.} Finally, the output of $\hat{u}_{N,\theta}$ is the projection of $\mathbf{z_{N+1}}$ by the local transformation $\mathcal{Q}$ as \citep{li2020fourier}:
\begin{equation}\label{eqn_output}
\begin{aligned}
    \hat{u}_{N,\theta}&=\mathcal{Q}\left(\sum_{i=1}^{N+1}\left( \sum_{n=0}^{M-1} \mathbf{z_i}[n]\overline{\mathscr{B}_i\big(e^{i2\pi n/M};a_{1:i}\big)}\right)\odot\mathscr{B}_i\right).\\
\end{aligned}
\end{equation}

\section{Properties of AMO}

\paragraph{Connections to AFD theory.} Adaptive Fourier decomposition (AFD) is a novel signal decomposition technique that leverages the Takenaka-Malmquist system and adaptive orthogonal bases \citep{qian2010intrinsic,qian2012adaptive}. It admits a proved convergence of any signal $s\in\mathcal{H}$ such that $s=\sum_{i=1}^\infty\langle s,\mathscr{B}_i\rangle\mathscr{B}_i$ \citep{qian2011algorithm,wang2022adaptive} for the chosen orthonormal bases $\mathscr{B}_i$ \citep{saitoh2016theory}. Thus, the output of Equation \ref{eqn_aggre} $\mathbf{z_{i+1}}$, is equivalent to the AFD operation, i.e., $\mathbf{z_{i+1}} = \sum_{k=1}^{i} \langle\mathbf{z_k},\mathscr{B}_k\rangle\mathscr{B}_k$. Furthermore, the output in Equation \ref{eqn_output} can be approximated as $\hat{u}_{N,\theta}=\mathcal{Q}\left( \sum_{i=1}^{N+1} \langle\mathbf{z_i},\mathscr{B}_i\rangle\mathscr{B}_i\right)\approx\sum_{i=1}^{N+1} \langle \hat{u}_{i-1,\theta},\mathscr{B}_i\rangle\mathscr{B}_i$, where $\hat{u}_{i-1,\theta}=\mathcal{Q}(\mathbf{z_i})$. This is also equivalent to the AFD operation. Thus, several theoretical properties of AMO, including convergence and error bound (see theorems and proofs in Appendix \ref{appen_theore}), can be guaranteed with efficiently large layers, thanks to AMO's deep connections with AFD theory.

\paragraph{Connections to \citet{parnichkun2024state}.} 
\citet{parnichkun2024state} proposed a state-free inference of SSMs by learning the coefficients of the rational transfer function $H$ instead of the traditional state-space matrices $A,B$, and $C$ \citep{gu2023mamba}, which is called rational transfer function (RTF) approach. Specifically, the RTF learns $H$ as:
\begin{equation}\label{eqn_rtf}
  H(z)=
  h_0 +
  \frac{b_1 z^{-1} + b_2 z^{-2} + \cdots + b_n z^{-n}}
       {1 + a_1 z^{-1} + a_2 z^{-2} + \cdots + a_n z^{-n}}\,,
\end{equation}
where $a_i$, $b_i$, and $h_0$ are denominator coefficients, numerator coefficients, and feedthrough term, respectively. When it comes to AMO, we push the formulation of transfer function in Equation \ref{eqn_tran} and learn the rational transfer function by learning the poles $a_{1:n}$ (for $n$ terms). In Appendix \ref{appen_trans}, we show that our way of learning poles leads to a similar form of Equation \ref{eqn_rtf} with $n$ learned parameters (poles) as opposed to learning $2n+1$ parameters in RTF. 

\paragraph{Computational complexity.} In terms of computational complexity, AMO has an overall computational complexity of $\mathcal{O}\!\big(N(M\log M+MD)\big)+\mathcal{O}(N_sMD)$. The former is from the processing block, whereas the latter comes from $\mathcal{P}$ and $\mathcal{Q}$. When $M$ is treated as a constant with $M\ll N_s$ and a local decoder is used, the dominant cost reduces to $\mathcal{O}(N_sD)+\mathcal{O}(N\,M\log M)$. Consequently, the complexity grows linearly with the number of mesh points $N_s$. With mesh size fixed, it is approximately linear in the number of latent tokens $M$ and the number of blocks $N$.

\section{Numerical Experiments}
To illustrate the effectiveness of AMO, we conduct numerical experiments with multiple baseline neural operators on diverse datasets including three categories: (i) regular grids: 2-D Darcy flow equation and 2-D Navier-Stokes equation \citep{li2020fourier}, (ii) irregular geometries: plasticity, airfoil, pipe, and elasticity \citep{li2023fourier}, (iii) PDEs with singularities: European option pricing under the Black-Scholes equation, and 3-D Brusselator (reaction-diffusion) equation from \citet{cao2024laplace} (see Appendix \ref{appen_mex}).

\paragraph{Metric.} In the training and evaluation stage, we utilize relative $L^2$ error as the metric for accuracy for all problems:
\begin{equation}
    \text{Rel-}L^2=\frac{1}{\mathcal{N}}\sum_{i=1}^\mathcal{N}\frac{||\mathcal{G}_{\theta}(a_i)-\mathcal{G}(a_i)||_{L^2}}{||\mathcal{G}(a_i)|_{L^2}},
\end{equation}
where $\mathcal{N}$ denotes the number of samples. We also consider training time, the number of parameters, and/or GPU memory usage as metrics for computational efficiency. 

\paragraph{Implementation details.} For baselines, we follow the implementation settings of their works. Note that the architecture of FNO \citep{li2020fourier} has been updated after publication, we evaluate FNO using the newest architecture. For AMO, we train $500$ epochs on all datasets. We use AdamW optimizer with decoupled weight decay $1\times10^{-5}$, base learning rate $2\times10^{-4}$, and a cosine decay schedule \citep{loshchilov2017decoupled} with a linear warm-up over the first $10\%$ of total steps. The nonlinearity is GELU inside the processing blocks. We clip global grad-norm at $0.5$ each step. Unless stated otherwise, we use batch size $16$, latent width $128$, $64$ latent tokens, $32$ adaptive poles, and 4 processing blocks with SSM state size $16$, depthwise 1-D convolution (per channel) of kernel size 4, channel expansion ratio 2. Experiments are conducted on a Linux workstation running Ubuntu (kernel 6.14, glibc 2.39) with Python 3.13.5 (Anaconda), PyTorch 2.8.0+cu129 (CUDA 12.9), an AMD Ryzen 9 9950X (16-core) processor, and a single NVIDIA GeForce RTX 4090 (48 GB) GPU. CUDA is enabled.

\subsection{Numerical results of benchmark datasets}

Table \ref{tab_operator-bench} shows the comprehensive comparison with various baselines on the six benchmark problems. Among those problems, N-S and Darcy flow datasets apply regular grids, elasticity dataset uses point clouds, whereas others are generated under structured meshes \citep{li2020fourier,li2023fourier}. AMO consistently outperforms existing SOTA models by an average improvement of $28.42\%$. In particular, for airfoil, Darcy, and N-S datasets, the relative $L^2$ error decreased more than $30\%$ compared to the existing SOTA models, demonstrating the superior performance of AMO compared to existing frequency-, transformer-, and Mamba-based models when solving complex dynamics and handling irregular geometries. To solve the complex dynamics, \citet{tiwari2025latent} incorporates latent representations and SSMs, which can be considered as integral kernels without orthogonality. Meanwhile, ONO \citep{xiao2023improved} uses an orthogonal attention to ensure orthogonality. Numerical results on irregular geometries, including elasticity ($0.0050\to0.0043$), plasticity ($0.0007\to0.0006$), airfoil ($0.0041\to0.0020$), and pipe ($0.0026\to0.0023$), show that the systematic integration of orthonormal kernels and SSMs leads to an exact AFD approximation and in turn improves PDE solution accuracy in irregular geometries.  

\begin{table}[htp]
\centering
\caption{Relative $L^2$ error comparisons of AMO with baselines across six benchmark datasets. Lower relative $L^2$ error is better. We quantify the improvement as the gain of AMO relative to the $L^2$ error of the second best model. \textbf{Bold} means the best model, \underline{underline} means the second best model, \textcolor{red}{red} means the third best model, and \textcolor{blue}{blue} means the fourth best model.}
\label{tab_operator-bench}
\begin{adjustbox}{width=\columnwidth}
\begin{tabular}{l|cccccc}
\toprule
\textbf{Models} & \textbf{Elasticity} & \textbf{Plasticity} & \textbf{Airfoil} & \textbf{Pipe} & \textbf{N-S} & \textbf{Darcy} \\
\midrule
FNO \citep{li2020fourier} & 0.0229 & 0.0074 &0.0138 &0.0067&\underline{0.0417}& \textcolor{blue}{0.0052}\\
U-FNO \citep{wen2022u} &  0.0239 &  0.0039 &  0.0269 & 0.0056 & 0.2231& 0.0183\\
F-FNO \citep{tran2021factorized} & 0.0263 & 0.0047 &  0.0078 &  0.0070 &0.2322 &0.0077\\
LNO \citep{wang2024latent} &  \textcolor{red}{0.0052} & 0.0029 & \textcolor{red}{0.0051} &\underline{0.0026} & \textcolor{blue}{0.0845} & \textcolor{red}{0.0049}\\
ONO \citep{xiao2023improved}& 0.0118 &  0.0048 & 0.0061 & 0.0052 & 0.1195 & 0.0076\\
WMT \citep{gupta2021multiwavelet} & 0.0359 & 0.0076 &  0.0075 &  0.0077 & 0.1541 & 0.0082\\
Galerkin \citep{cao2021choose} & 0.0240 & 0.0120 & 0.0118 & 0.0098 & 0.1401 & 0.0084\\
LSM \citep{wu2023solving}&  0.0218 & 0.0025 & 0.0059 & 0.0050 & 0.1535 & 0.0065\\
OFormer \citep{li2022transformer}& 0.0183 & 0.0017 & 0.0183 & 0.0168& 0.1705&0.0124\\
Transolver \citep{wu2024transolver}   & \textcolor{blue}{0.0062} & \textcolor{red}{0.0013} & 0.0053 & 0.0047 & 0.0879 & 0.0059 \\
Transolver++ \citep{luo2025transolver++} & 0.0064 & \textcolor{blue}{0.0014} & \textcolor{red}{0.0051} & \textcolor{blue}{0.0027} & 0.1010 & 0.0089 \\
LAMO \citep{tiwari2025latent}  & \underline{0.0050} & \underline{0.0007} & \underline{0.0041} & \textcolor{red}{0.0038} & \textcolor{red}{0.0460} & \underline{
0.0039} \\
\textbf{AMO (ours)} & \textbf{0.0043} & \textbf{0.0006} & \textbf{0.0020} & \textbf{0.0023} & \textbf{0.0278} & \textbf{0.0021} \\
\midrule
\textbf{Improvement} & 14.0\% & 14.3\% & 51.2\% & 11.5\% & 33.3\% & 46.2\% \\
\bottomrule
\end{tabular}
\end{adjustbox}
\end{table}

\paragraph{Computational Efficiency.} To explore the computational efficiency of AMO, we focus on Darcy and airfoil problems. On average, AMO reaches $46.2\%$ and $51.2\%$ reduction in training time over SOTA models in these two problems, as shown in Figure \ref{fig_efficiency}. With light architectures and small GPU memory, AMO achieves the best training speed. Compared to the SOTA neural operator, LaMO \citep{tiwari2025latent}, AMO is $\sim1.2\times$ faster and $\sim2.5\times$ lighter with similar GPU memory. Instead of using orthogonal attention as in ONO \citep{xiao2023improved}, AMO employs bases in the orthogonal form (Equation \ref{eqn_tm}), which does not require an orthogonalization process, thereby saving $\sim2.7\times$ in training time and $\sim3\times$ in GPU memory compared to ONO.

\paragraph{Scalability.} We examine the computational scalability of AMO on 2-D Darcy flow problem. From Table \ref{tab_scalability}, we observe that, as the grid dimension changes from 64 to 128 ($N_s$ becomes 4 times larger), both training and inference times increase approximately linearly (by about 4 times), which aligns with the computational complexity result mentioned earlier. The memory usage remains relatively constant with only a slight increase. This reflects the architectural characteristics of AMO, where the main computations (SSM blocks) are performed on $M$ latent tokens rather than on $N_s$ physical points, and thus the memory footprint is largely decoupled from the input resolution $N_s$.
\begin{table}[htp]
\centering
\caption{AMO is computationally scalable with respect to input resolution $N_s$.}
\label{tab_scalability}
\begin{adjustbox}{width=\columnwidth}
\begin{tabular}{cccccc}
\toprule
\textbf{Grid dimensions} & \textbf{Grid size $N_s$} & \textbf{Training time (sec/epoch)} & \textbf{Inference time (sec/epoch)} & \textbf{GPU memory (GB)} \\
\midrule
$64\times64$ & $ 4096$ & $14.0$  & $0.007$ & $2.3$ \\
$128\times128$ & $ 16384 $ & $52.5$  & $0.28$ &$2.4$\\
$256\times256$ & $ 65536$ & $205.0$  & $1.12$ & $2.7$\\
\bottomrule
\end{tabular}
\end{adjustbox}
\end{table}

\paragraph{Learned pole distributions across layers.} To understand how the adaptive poles are selected and evolved, Figures \ref{fig:2dpole} and \ref{fig:3dpole} showcase the distributions per layer for 2-D Darcy flow and 3-D Brusselator equations. The learned poles of AMO on Darcy flow problem tend to approach to the boundary of the unit disk, while those on the Brusselator problem tend to be in the interior of the unit disk. The reason is that,  the challenging characteristics and singularities of the Darcy flow problem are located at the boundaries, and then more adaptive poles would be put there. Meanwhile, the complexity of the Brusselator problem does not come from the boundaries. It comes from the local, non-linear reaction that happens at every single point inside the domain. Therefore, most of the learned poles should be put inside the unit disk.

\subsection{European Options Pricing}
To demonstrate the versatility of AMO in solving different PDEs in different contexts, we consider the European calls/puts problem modeled using the Black–Scholes equation with continuous dividend yield $q$. For contract/market parameters $(r,\sigma,q,K,T,\texttt{is\_call})$, the price $V(S,t)$ satisfies the Black–Scholes equation \citep{barles1998option}:
\begin{equation}
\partial_t V+\tfrac12\sigma^2 S^2\,\partial_{SS}V+(r-q)S\,\partial_S V-rV=0,\quad S\in[S_{\min},S_{\max}],\ t\in[0,T],
\end{equation}
with terminal payoff $V(S,T)=\max(\pm(S-K),0)$ ($+$ sign for calls, $-$ for puts) and the linear boundary conditions $V(0,t)=0$ for calls, $V(0,t)=K e^{-r(T-t)}$ for puts, and controlled growth as $S\to\infty$. This problem setting leads to two singular features: (i) the terminal payoff kink at $S=K$ (jump in $\partial_S V$, concentration in $\partial_{SS}V$) as $t_{\text{norm}}\!\uparrow 1$; and (ii) degeneracy near small $S$ as a result of the $S^2\partial_{SS}V$ diffusion term. Our goal is to learn the operator that maps the parameters $(r,\sigma,q,K,T,\texttt{is\_call})$ to the price $V(S,t)$. By comparing AMO with a set of top-performing solvers, we observe from Table \ref{tab_euro_option_resources} that average improvements of $25\%$, $4.1\%$, and $52.7\%$ have been achieved by AMO in terms of relative $L^2$ error, training time, and parameter counts, respectively. This indicates that AMO can accurately and efficiently solve PDE problems with singular features.

\begin{figure}[ht!]
    \centering
    \includegraphics[width=\linewidth]{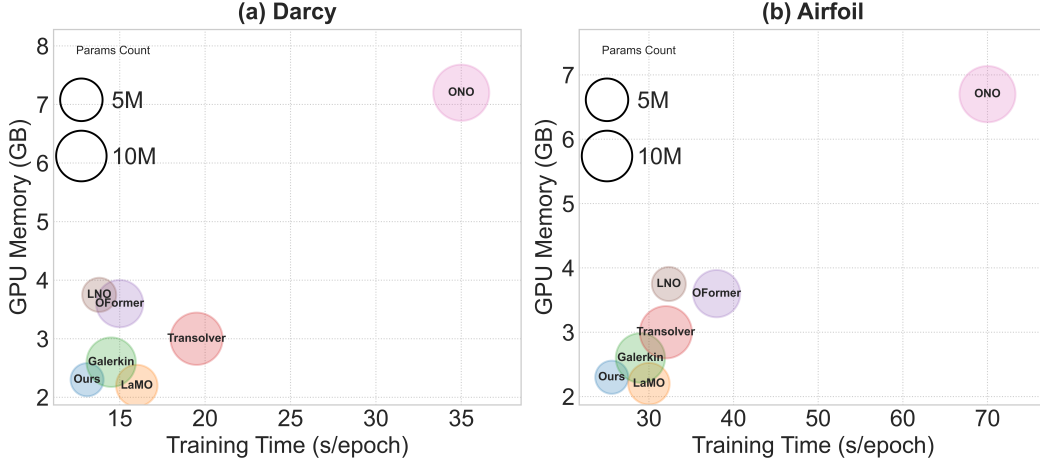}
    \caption{Comparisons of training time per epoch, number of parameters, and GPU memory among existing SOTA models on (a) Darcy and (b) airfoil, where AMO exhibits the strongest incremental gains.}
    \label{fig_efficiency}
\end{figure}
\begin{table}[htp]
\centering
\caption{European option pricing: relative $L^2$ error and resource profile. Lower is better for error, GPU memory, and training time. Parameter counts shown in millions. \textbf{Bold} = best, \underline{underline} = second best, and \textcolor{red}{red} = third best.}
\label{tab_euro_option_resources}

\begin{adjustbox}{width=\columnwidth}
\begin{tabular}{l|c|c|c}
\toprule
\textbf{Models}
& \textbf{Rel.\ $L^2$\,($\downarrow$)}
& \textbf{Training Time (sec/epoch, $\downarrow$)}
& \textbf{Params (M, $\downarrow$)} \\
\midrule
FNO \citep{li2020fourier}           & 0.0016 &  25.1 & 3.78 \\
LNO \citep{wang2024latent}          & \textcolor{red}{0.0010} &  \underline{21.7} & \underline{2.56} \\
Transolver \citep{wu2024transolver} & 0.0012 &  \textcolor{red}{22.3} & 5.91 \\
LAMO \citep{tiwari2025latent}       & \underline{0.0008} &  22.5 & \textcolor{red}{3.52} \\
\textbf{AMO (ours)}                & \textbf{0.0006} &  \textbf{20.8} & \textbf{1.21} \\
\bottomrule
\end{tabular}
\end{adjustbox}
\end{table}

\subsection{Ablation studies}
\paragraph{Adaptive kernels vs. static kernels.} We now consider the need and benefits of using adaptive kernels. A kernel is adaptive when its parameterization (e.g., coefficients) varies with the input. In this work, the formulation of Equation \ref{eqn_tm} varies with the learned poles $a_{1:i}$ and thus is an adaptive kernel. We also randomly fix the value of $a_{1:i}$ for static kernels for comparison. Furthermore, although a total of $i$ poles are needed for $i$-th processing block, one can still identify more poles and select the best $i$ poles for implementation. Table \ref{tab_fixed_adaptive_bench} shows the relative $L^2$ error results across six benchmark datasets and the European options (EO) dataset. We find that, using adaptive kernels, the relative $L^2$ errors reduce significantly compared to using static poles for all benchmark problems considered. In fact, the relative $L^2$ errors when selecting only 4 poles are lower than those when selecting 32 static poles.
\begin{table}[htp]
\centering
\caption{Relative $L^2$ error comparisons for \textbf{Static} vs.\ \textbf{Adaptive} kernels across seven benchmarks. Lower is better.}
\label{tab_fixed_adaptive_bench}
\begin{adjustbox}{width=\columnwidth}
\begin{tabular}{l |c | ccccccc}
\toprule
\textbf{Models} & \textbf{Number of poles} & \textbf{Elasticity} & \textbf{Plasticity} & \textbf{Airfoil} & \textbf{Pipe} & \textbf{N-S} & \textbf{Darcy} & \textbf{EO} \\
\midrule
AMO (static)  & 32 & 0.0097 & 0.0021 & 0.0067 & 0.0072 & 0.1103 & 0.0174 & 0.0035\\
\midrule
\multirow{7}{*}{AMO (adaptive)}
 & 4   & 0.0056 & 0.0012 & 0.0033 & 0.0029 & 0.0311 & 0.0057 & 0.0014\\
 & 6   & 0.0051 & 0.0010 & 0.0031 & 0.0027 & 0.0298 & 0.0047 & 0.0010\\
 & 8   & 0.0049 & 0.0008 & 0.0027 & 0.0025 & 0.0281 & 0.0036 & 0.0009\\
 & 16  & 0.0046 & 0.0008 & 0.0023& 0.0028 & 0.0290 & 0.0029 & 0.0008\\
 & 32  &\textbf{0.0043} & \textbf{0.0006} & \textbf{0.0020} & \textbf{0.0023} & \textbf{0.0278} & \textbf{0.0021} & \textbf{0.0006}\\
 & 64  & 0.0048 & 0.0007 & 0.0036 & 0.0031 & 0.0372 & 0.0046 & 0.0009\\
\bottomrule
\end{tabular}
\end{adjustbox}
\end{table}

\paragraph{Need for ensuring orthogonality.} To understand how orthogonal kernels affect AMO performance, we conduct another ablation study by using non-orthogonal kernels (i.e., Equation \ref{eqn_ker}) in the AMO framework. In this case, the transfer functions used in SSMs are $H_i(e^{i\omega})= \overline{(1 - |a_i|^2)\sum_{n=0}^\infty (\overline{a_i})^n e^{in\omega}}$ to match the output of AFD operation. Without orthogonality, AMO experiences higher relative $L^2$ error, especially for problems with irregular geometries (e.g., airfoil $0.0020\to0.0083$ and elasticity $0.0043\to0.0094$). At the same time, the training time also increases by $\sim50.3\%$ per epoch on average across all six benchmark datasets. This shows that the use of orthogonal kernels (i.e., TM systems) helps improve both accuracy and computational efficiency of AMO solver.

\paragraph{Choice of SSMs.} Finally, we evaluate the choice of bidirectional SSMs in AMO compared to unidirectional SSMs and multidirectional SSMs. Results in Figure \ref{fig_ssm} indicate that the choice of bidirectional SSMs in AMO consistently outperforms other two SSMs in all datasets. 

\begin{figure}[ht!]
    \centering
    \includegraphics[width=0.9\linewidth]{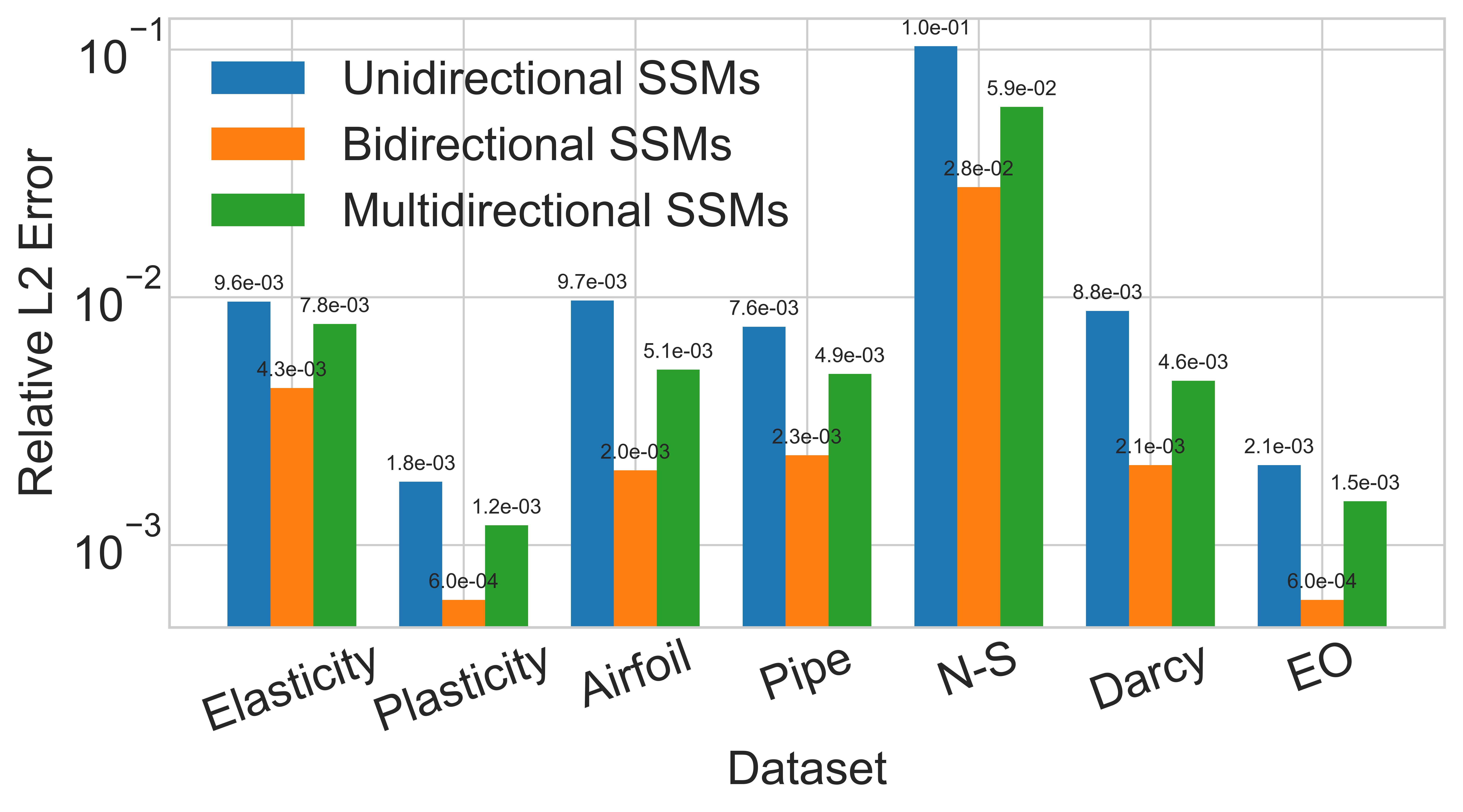}
    \caption{Contribution of three SSMs across seven benchmark datasets. Note that we do not apply weights shared for all experiments. Lower is better.}
    \label{fig_ssm}
\end{figure}

\subsection{Experiment using real-world noisy dataset}
To validate AMO's performance on noisy real-world datasets, we perform experiments using the latex glove DIC (Digital Image Correlation) original dataset \citep{ifno}. The goal is to learn the mechanical response of a nitrile glove sample directly from experimental data, without assuming a known constitutive law. The goal is to predict the displacement field at the current loading step. The input includes the spatial coordinates, the displacement field from the previous step, and the current boundary displacement. We compare the performance of AMO to the current SOTA of this dataset, IFNO, as well as FNO as follows. To ensure fair comparison, we conduct experiments using the same settings as IFNO with the number of hidden layers ranging from 3 to 12.\\
\begin{table}[htbp]
\centering
\caption{Relative \(L^2\) error of AMO and other baselines using the latex glove DIC (Digital Image Correlation) original dataset.}
\begin{adjustbox}{width=\columnwidth}
\begin{tabular}{c c c c}
\toprule
\bf Number of hidden layers & \bf AMO & \bf IFNO & \bf FNO \\
\midrule
3  &  2.87E-02 $\pm$ 4.29E-04  &  3.43E-02 $\pm$ 4.96E-04  &  3.40E-02  $\pm$ 4.09E-04 \\
6   &  2.50E-02 $\pm$ 3.28E-04  &  3.34E-02 $\pm$ 4.53E-04  &  3.84E-02 $\pm$ 4.21E-04 \\
12  &  2.32E-02 $\pm$ 4.20E-04  &  3.32E-02 $\pm$ 4.41E-04 &  4.66E-02 $\pm$ 1.47E-03 \\
\bottomrule
\end{tabular}
\end{adjustbox}
\end{table}\\

In addition, \citet{ifno} also reported the results of generalized Mooney-Rivlin (GMR) model in two settings. The relative $L^2$ errors of GMR model fitting and GMR inverse analysis are 3.30E-01 and 2.91E-01, respectively. We can observe that our AMO consistently outperforms other models in every $L$. Finally, the best reported result of IFNO is 3.30E-02 $\pm$ 4.63E-04 when $L=24$ \citep{ifno}. Although we do not conduct the experiment $L=24$ due to the limited time, our AMO still performs better than the best result of IFNO.

\section{Conclusions}
In this paper, we propose a novel neural operator AMO for solving nonlinear PDEs on irregular geometries and singularities. AMO maps the physical tokens in an RKHS where the global spectral transform and data-dependent orthogonal kernels are incorporated. By conducting a tailored design of the TM layer and SSM block fully guided by the AFD theory, we show that the output of AMO exactly matches with AFD oepration, hence offering rigorous convergence guarantee and other desirable properties. We show that the novel architecture of AMO enables its outstanding performance compared to existing SOTA neural operators in a series of physical and financial benchmark problems.

\section{Reproducibility Statement}
All code and datasets have been either made publicly available in an anonymous repository or as a part of supplementary material to facilitate replication and verification. The experimental setup, including training steps, model configurations, and hardware details, is described in detail in the paper. We have also provided a full description of implementation details, to assist others in reproducing our experiments. Additionally, six benchmark datasets, such as pipe, are publicly available, ensuring consistent and reproducible evaluation results.

\bibliography{iclr2026_conference}

@article{tripura2023wavelet,
  title={Wavelet neural operator for solving parametric partial differential equations in computational mechanics problems},
  author={Tripura, Tapas and Chakraborty, Souvik},
  journal={Computer Methods in Applied Mechanics and Engineering},
  volume={404},
  pages={115783},
  year={2023},
  publisher={Elsevier}
}

@article{li2020fourier,
  title={Fourier neural operator for parametric partial differential equations},
  author={Li, Zongyi and Kovachki, Nikola and Azizzadenesheli, Kamyar and Liu, Burigede and Bhattacharya, Kaushik and Stuart, Andrew and Anandkumar, Anima},
  journal={arXiv preprint arXiv:2010.08895},
  year={2020}
}

@article{gupta2021multiwavelet,
  title={Multiwavelet-based operator learning for differential equations},
  author={Gupta, Gaurav and Xiao, Xiongye and Bogdan, Paul},
  journal={Advances in Neural Information Processing Systems},
  volume={34},
  pages={24048--24062},
  year={2021}
}

@article{rahman2022u,
  title={U-no: {U}-shaped neural operators},
  author={Rahman, Md Ashiqur and Ross, Zachary E and Azizzadenesheli, Kamyar},
  journal={arXiv preprint arXiv:2204.11127},
  year={2022}
}

@inproceedings{fanaskov2023spectral,
  title={Spectral neural operators},
  author={Fanaskov, Vladimir Sergeevich and Oseledets, Ivan V},
  booktitle={Doklady Mathematics},
  volume={108},
  pages={S226--S232},
  year={2023},
  organization={Springer}
}

@article{wu2023solving,
  title={Solving high-dimensional {PDEs} with latent spectral models},
  author={Wu, Haixu and Hu, Tengge and Luo, Huakun and Wang, Jianmin and Long, Mingsheng},
  journal={arXiv preprint arXiv:2301.12664},
  year={2023}
}

@article{li2023fourier,
  title={Fourier neural operator with learned deformations for {PDEs} on general geometries},
  author={Li, Zongyi and Huang, Daniel Zhengyu and Liu, Burigede and Anandkumar, Anima},
  journal={Journal of Machine Learning Research},
  volume={24},
  number={388},
  pages={1--26},
  year={2023}
}

@article{chen2024learning,
  title={Learning neural operators on {R}iemannian manifolds},
  author={Chen, Gengxiang and Liu, Xu and Meng, Qinglu and Chen, Lu and Liu, Changqing and Li, Yingguang},
  journal={National Science Open},
  volume={3},
  number={6},
  pages={20240001},
  year={2024},
  publisher={China Science Publishing \& Media Ltd. and EDP Sciences}
}

@article{lingsch2023beyond,
  title={Beyond regular grids: {F}ourier-based neural operators on arbitrary domains},
  author={Lingsch, Levi and Michelis, Mike Y and De B{\'e}zenac, Emmanuel and Perera, Sirani M and Katzschmann, Robert K and Mishra, Siddhartha},
  journal={arXiv preprint arXiv:2305.19663},
  year={2023}
}

@article{tiwari2025latent,
  title={Latent Mamba Operator for Partial Differential Equations},
  author={Tiwari, Karn and Dutta, Niladri and Krishnan, NM and others},
  journal={International Conference on Machine Learning},
  year={2025}
}

@article{gu2023mamba,
  title={Mamba: Linear-time sequence modeling with selective state spaces},
  author={Gu, Albert and Dao, Tri},
  journal={arXiv preprint arXiv:2312.00752},
  year={2023}
}

@article{gu2021efficiently,
  title={Efficiently modeling long sequences with structured state spaces},
  author={Gu, Albert and Goel, Karan and R{\'e}, Christopher},
  journal={arXiv preprint arXiv:2111.00396},
  year={2021}
}

@article{parnichkun2024state,
  title={State-free inference of state-space models: The transfer function approach},
  author={Parnichkun, Rom N and Massaroli, Stefano and Moro, Alessandro and Smith, Jimmy TH and Hasani, Ramin and Lechner, Mathias and An, Qi and R{\'e}, Christopher and Asama, Hajime and Ermon, Stefano and others},
  journal={arXiv preprint arXiv:2405.06147},
  year={2024}
}

@article{qian2011algorithm,
  title={Algorithm of adaptive Fourier decomposition},
  author={Qian, Tao and Zhang, Liming and Li, Zhixiong},
  journal={IEEE Transactions on Signal Processing},
  volume={59},
  number={12},
  pages={5899--5906},
  year={2011},
  publisher={IEEE}
}

@article{wang2022adaptive,
  title={Adaptive Fourier decomposition for multi-channel signal analysis},
  author={Wang, Ze and Wong, Chi Man and Rosa, Agostinho and Qian, Tao and Wan, Feng},
  journal={IEEE Transactions on Signal Processing},
  volume={70},
  pages={903--918},
  year={2022},
  publisher={IEEE}
}

@article{loshchilov2017decoupled,
  title={Decoupled weight decay regularization},
  author={Loshchilov, Ilya and Hutter, Frank},
  journal={arXiv preprint arXiv:1711.05101},
  year={2017}
}

@article{wu2024transolver,
  title={Transolver: A fast transformer solver for pdes on general geometries},
  author={Wu, Haixu and Luo, Huakun and Wang, Haowen and Wang, Jianmin and Long, Mingsheng},
  journal={arXiv preprint arXiv:2402.02366},
  year={2024}
}

@article{tran2021factorized,
  title={Factorized fourier neural operators},
  author={Tran, Alasdair and Mathews, Alexander and Xie, Lexing and Ong, Cheng Soon},
  journal={arXiv preprint arXiv:2111.13802},
  year={2021}
}

@article{wang2024latent,
  title={Latent neural operator for solving forward and inverse pde problems},
  author={Wang, Tian and Wang, Chuang},
  journal={Advances in Neural Information Processing Systems},
  volume={37},
  pages={33085--33107},
  year={2024}
}

@article{xiao2023improved,
  title={Improved operator learning by orthogonal attention},
  author={Xiao, Zipeng and Hao, Zhongkai and Lin, Bokai and Deng, Zhijie and Su, Hang},
  journal={arXiv preprint arXiv:2310.12487},
  year={2023}
}

@article{cao2021choose,
  title={Choose a transformer: Fourier or galerkin},
  author={Cao, Shuhao},
  journal={Advances in neural information processing systems},
  volume={34},
  pages={24924--24940},
  year={2021}
}

@article{li2022transformer,
  title={Transformer for partial differential equations' operator learning},
  author={Li, Zijie and Meidani, Kazem and Farimani, Amir Barati},
  journal={arXiv preprint arXiv:2205.13671},
  year={2022}
}

@article{wen2022u,
  title={U-FNO—An enhanced Fourier neural operator-based deep-learning model for multiphase flow},
  author={Wen, Gege and Li, Zongyi and Azizzadenesheli, Kamyar and Anandkumar, Anima and Benson, Sally M},
  journal={Advances in Water Resources},
  volume={163},
  pages={104180},
  year={2022},
  publisher={Elsevier}
}

@article{barles1998option,
  title={Option pricing with transaction costs and a nonlinear {Black-Scholes} equation},
  author={Barles, Guy and Soner, Halil Mete},
  journal={Finance and Stochastics},
  volume={2},
  number={4},
  pages={369--397},
  year={1998},
  publisher={Springer}
}

@inproceedings{gupta2022non,
  title={Non-linear operator approximations for initial value problems},
  author={Gupta, Gaurav and Xiao, Xiongye and Balan, Radu and Bogdan, Paul},
  booktitle={International Conference on Learning Representations (ICLR)},
  year={2022}
}

@article{xiao2023coupled,
  title={Coupled multiwavelet neural operator learning for coupled partial differential equations},
  author={Xiao, Xiongye and Cao, Defu and Yang, Ruochen and Gupta, Gaurav and Liu, Gengshuo and Yin, Chenzhong and Balan, Radu and Bogdan, Paul},
  journal={arXiv preprint arXiv:2303.02304},
  year={2023}
}

@article{yu2024nonlocal,
  title={Nonlocal attention operator: Materializing hidden knowledge towards interpretable physics discovery},
  author={Yu, Yue and Liu, Ning and Lu, Fei and Gao, Tian and Jafarzadeh, Siavash and Silling, Stewart A},
  journal={Advances in Neural Information Processing Systems},
  volume={37},
  pages={113797--113822},
  year={2024}
}

@article{zheng2024alias,
  title={Alias-free mamba neural operator},
  author={Zheng, Jianwei and Li, Wei and Xu, Ni and Zhu, Junwei and Zhang, Xiaoqin},
  journal={Advances in Neural Information Processing Systems},
  volume={37},
  pages={52962--52995},
  year={2024}
}

@article{cheng2024mamba,
  title={Mamba neural operator: Who wins? transformers vs. state-space models for pdes},
  author={Cheng, Chun-Wun and Huang, Jiahao and Zhang, Yi and Yang, Guang and Sch{\"o}nlieb, Carola-Bibiane and Aviles-Rivero, Angelica I},
  journal={arXiv preprint arXiv:2410.02113},
  year={2024}
}

@article{hu2024state,
  title={State-space models are accurate and efficient neural operators for dynamical systems},
  author={Hu, Zheyuan and Daryakenari, Nazanin Ahmadi and Shen, Qianli and Kawaguchi, Kenji and Karniadakis, George Em},
  journal={arXiv preprint arXiv:2409.03231},
  year={2024}
}

@article{qian2010intrinsic,
  title={Intrinsic mono-component decomposition of functions: an advance of {Fourier} theory},
  author={Qian, Tao},
  journal={Mathematical Methods in the Applied Sciences},
  volume={33},
  number={7},
  pages={880--891},
  year={2010},
  publisher={Wiley Online Library}
}

@article{qian2012adaptive,
  title={Adaptive {F}ourier decomposition of functions in quaternionic {H}ardy spaces},
  author={Qian, Tao and Spr{\"o}{\ss}ig, Wolfgang and Wang, Jinxun},
  journal={Mathematical Methods in the Applied Sciences},
  volume={35},
  number={1},
  pages={43--64},
  year={2012},
  publisher={Wiley Online Library}
}

@article{guibas2021adaptive,
  title={Adaptive {F}ourier neural operators: Efficient token mixers for transformers},
  author={Guibas, John and Mardani, Morteza and Li, Zongyi and Tao, Andrew and Anandkumar, Anima and Catanzaro, Bryan},
  journal={arXiv preprint arXiv:2111.13587},
  year={2021}
}

@book{saitoh2016theory,
  title={Theory of reproducing kernels and applications},
  author={Saitoh, Saburou and Sawano, Yoshihiro and others},
  volume={44},
  year={2016},
  publisher={Springer}
}

@article{ifno,
title = {Learning deep Implicit Fourier Neural Operators (IFNOs) with applications to heterogeneous material modeling},
journal = {Computer Methods in Applied Mechanics and Engineering},
volume = {398},
pages = {115296},
year = {2022},
author = {Huaiqian You and Quinn Zhang and Colton J. Ross and Chung-Hao Lee and Yue Yu},
}

@misc{luo2025transolver++,
      title={Transolver++: An Accurate Neural Solver for PDEs on Million-Scale Geometries}, 
      author={Huakun Luo and Haixu Wu and Hang Zhou and Lanxiang Xing and Yichen Di and Jianmin Wang and Mingsheng Long},
      year={2025},
      eprint={2502.02414},
      archivePrefix={arXiv},
      primaryClass={cs.LG},
      url={https://arxiv.org/abs/2502.02414}, 
}

@article{cao2024laplace,
  title={Laplace neural operator for solving differential equations},
  author={Cao, Qianying and Goswami, Somdatta and Karniadakis, George Em},
  journal={Nature Machine Intelligence},
  volume={6},
  number={6},
  pages={631--640},
  year={2024},
  publisher={Nature Publishing Group UK London}
}
\bibliographystyle{iclr2026_conference}

\appendix
\newpage

\section{Notation List}
\bgroup
\def\arraystretch{1.5}

\begin{tabular}{p{1.25in}p{3.75in}}
$\displaystyle a$ & Parameter function (input) \\
$\displaystyle \hat{u}_{N,\theta}$ & Output of AMO with $N$ blocks and parameters $\theta$ \\
$\displaystyle N$ & Number of processing blocks \\
$\displaystyle N_s$ & Number of input physical tokens \\
$\displaystyle M$ & Number of encoded latent tokens ($M \ll N_s$) \\
$\displaystyle D_{\text{embed}}$ & Embedding dimension of latent tokens \\
$\displaystyle \mathbf{x}_{\text{phys}}$ & Input physical features \\
$\displaystyle \mathbf{g}_{\text{phys}}$ & Positional embedding of coordinates \\
$\displaystyle \mathbf{z}_i$ & Token representation after the $i$-th block \\
$\displaystyle \mathbf{z}_0$ & Encoded tokens produced by the lifting operator $\mathcal{P}$ \\
$\displaystyle \mathbf{z}_1$ & Tokens mapped into RKHS by operator $\mathcal{R}$ \\
$\displaystyle \mathcal{P}$ & Lifting operator mapping physical tokens to encoded tokens \\
$\displaystyle \mathcal{Q}$ & Projection operator mapping latent tokens back to output space \\
$\displaystyle \mathcal{R}$ & Mapping operator from latent tokens to RKHS \\
$\displaystyle \mathcal{L}^i$ & Processing block at layer $i$ ($\mathrm{SSM}^i \circ \mathrm{TM}^i$) \\
$\displaystyle \mathcal{S}^i$ & Aggregation operator with skip connections at block $i$ \\
$\displaystyle \mathrm{TM}^i$ & TM layer performing spectral transform via TM bases \\
$\displaystyle \mathrm{SSM}^i$ & Bidirectional SSM block parameterized by transfer function \\
$\displaystyle \mathscr{B}_i(z;a_{1:i})$ & $i$-th TM basis generated by poles $a_{1:i}$ \\
$\displaystyle a_{1:i}$ & Set of learned poles $\{a_1,\ldots,a_i\}$ in the unit disk $\mathbb{D}$ \\
$\displaystyle K_a(z)$ & Reproducing kernel $\frac{1}{1-\overline{a}z}$\\
$\displaystyle H_i(e^{i\omega})$ & Transfer function of the $i$-th SSM block \\
$\displaystyle h_i[n]$ & Impulse response of the $i$-th SSM block \\
$\displaystyle \langle x,f\rangle$ & Inner product $\frac{1}{\Tilde{N}}\sum_{n=0}^{\Tilde{N}-1} x[n]\overline{f(e^{i2\pi n/\Tilde{N}})}$ \\
$\displaystyle \odot$ & Element-wise (Hadamard) product \\
$\displaystyle \mathcal{H}$ & Reproducing Kernel Hilbert Space (RKHS) \\
$\displaystyle \Tilde{N}$ & Length of signal in inner product definition \\
\end{tabular}
\vspace{0.25cm}

\section{Illustrative Examples}\label{appen_mex}
\paragraph{1-D advection PDE with high-frequency perturbation.}We evaluate LaMO on a   1-D linear advection benchmark governed by \begin{equation}
    u_t + c\,u_x = 0
\end{equation} on a periodic unit interval. Initial conditions $u_0(x)$ are synthesized as smooth Fourier mixtures $\sum_{k=1}^{k_{\max}} a_k \sin(2\pi k x + \phi_k)$ with amplitudes decaying as $a_k \sim (1+k)^{-1}$, to which we add a weak high-frequency spike at wavenumber $k_{\mathrm{hi}}$ to probe aliasing and phase accuracy. Trajectories are advanced to time $T$ with a conservative first-order upwind scheme at Courant number $\mathrm{CFL}=c\,\Delta t/\Delta x\le 0.5$, ensuring stability while preserving sharp phase relationships; the target is the advected field $u(\cdot,T)$.
\begin{figure}
    \centering
    \includegraphics[width=0.9\linewidth]{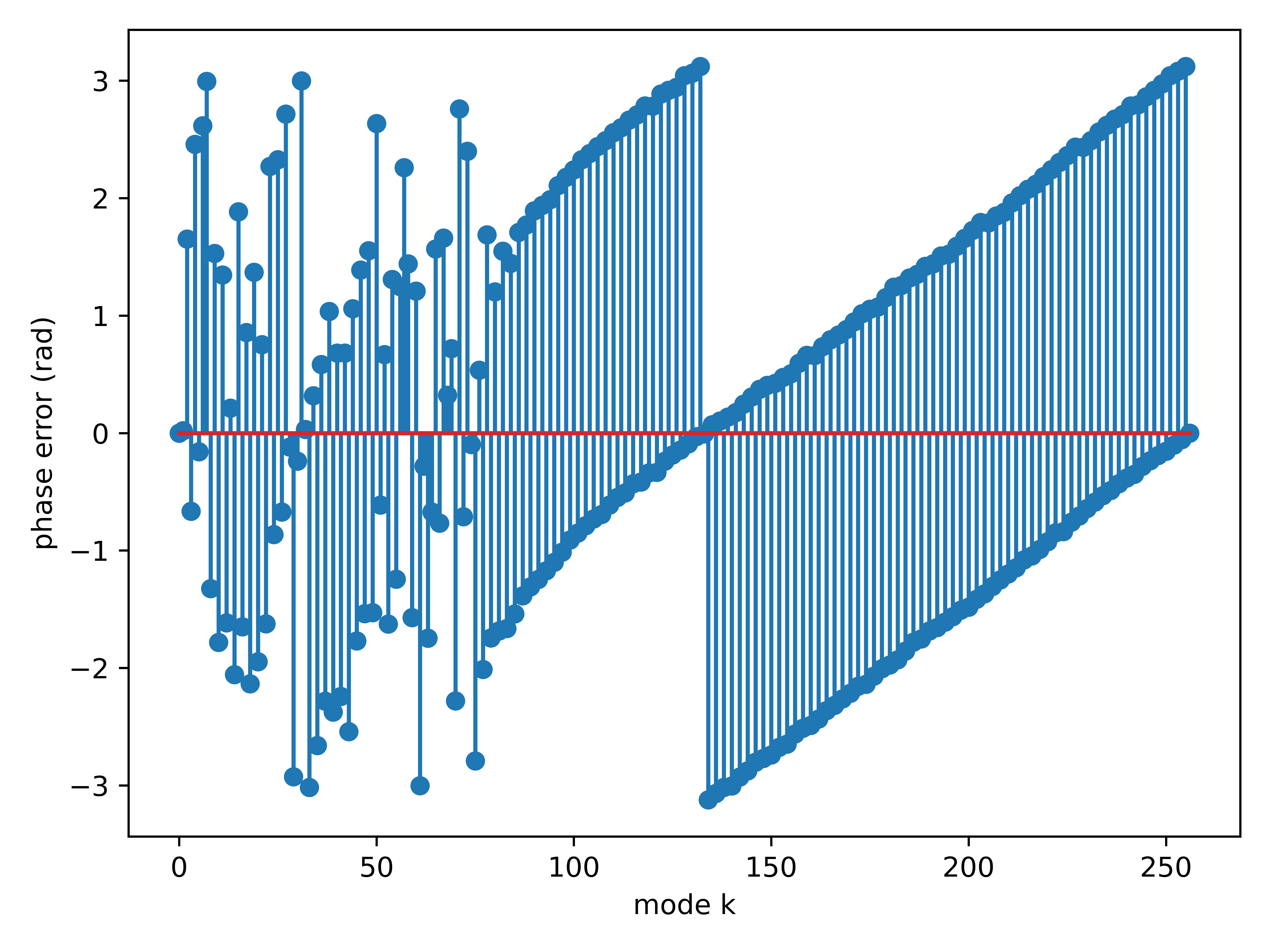}
    \caption{Phase error of solutions predicted by LaMO.}
    \label{fig:pe}
\end{figure}

Figure \ref{fig:pe} visualizes the phase error of LaMO’s predictions, revealing a pronounced degradation for high-frequency modes (approximately $k\in[140,\,250]$). This suggests that LaMO struggles to faithfully capture phase at the upper end of the spectrum.
\paragraph{2-D Darcy flow equation with fractal noise.} We construct a challenging 2-D Darcy dataset by solving 
\begin{equation}
 -\nabla\!\cdot\!\big(k(x,y)\nabla u(x,y)\big)=f(x,y)   
\end{equation} on $[0,1]^2$ with homogeneous Dirichlet boundaries, where the permeability $k$ is positive, highly heterogeneous, and fractal-like. Specifically, $k$ is generated by exponentiating a band-limited fractional Gaussian field (small Hurst parameter for roughness) and then modulating it with narrow channel masks and inclusions to induce strong anisotropy and high contrast. The forcing $f$ combines a weak background term with several randomized Gaussian sources/sinks, which produce near-singular behavior in the solution. The variable-coefficient elliptic problem is discretized on a Cartesian grid using a flux-conservative 5-point stencil with harmonic averaging of $k$, and solved to tight tolerance via conjugate gradients. For learning, each sample is subsampled irregularly: we draw $P$ points $\{(x_i,y_i)\}$ and record $u(x_i,y_i)$, yielding pairs $(\mathrm{XY},U)$ without exposing $k$ or $f$. 

To visualize and stress singular structures, we show in Figure \ref{fig_motiv} (a) and (c): (i) contours of the potential $u$ highlighting global flow topology, and (ii) a logarithmic map of the gradient magnitude, $\log |\nabla u|$, computed on a reconstructed dense grid via triangulation. Figure \ref{fig_motiv} shows LAMO cannot capture the singularities of $u$ and $\log |\nabla u|$. Furthermore, once the complex singularities appear, the performance of LAMO will be affected.  
\begin{figure}
  \centering
  \begin{subfigure}{0.48\linewidth}
    \centering
    \includegraphics[width=\linewidth, page=1]{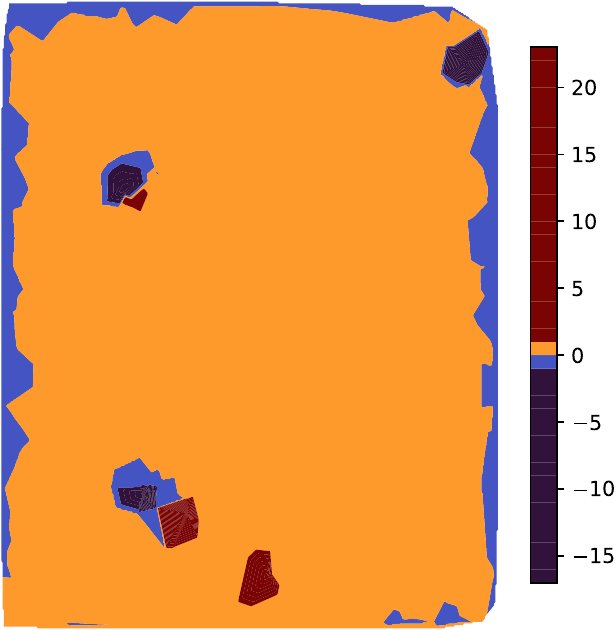}%
    \caption{Ground truth $u$}
  \end{subfigure}\hfill
  \begin{subfigure}{0.48\linewidth}
    \centering
    \includegraphics[width=\linewidth, page=1]{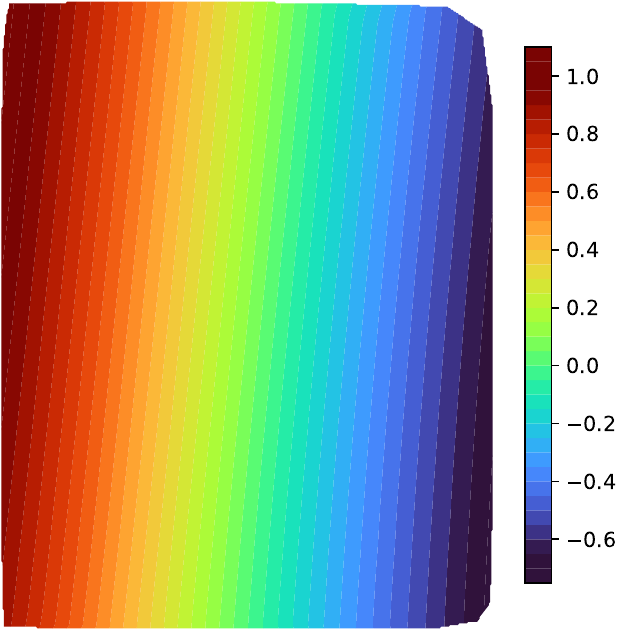}%
    \caption{Predicted by LaMO}
  \end{subfigure}

  \vspace{0.6em}

  \begin{subfigure}{0.48\linewidth}
    \centering
    \includegraphics[width=\linewidth, page=1]{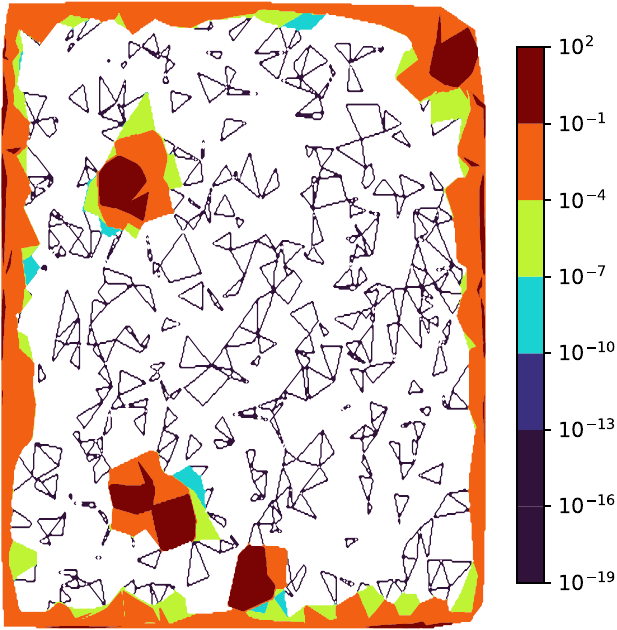}%
    \caption{Ground truth $\log |\nabla u|$}
  \end{subfigure}\hfill
  \begin{subfigure}{0.48\linewidth}
    \centering
    \includegraphics[width=\linewidth, page=1]{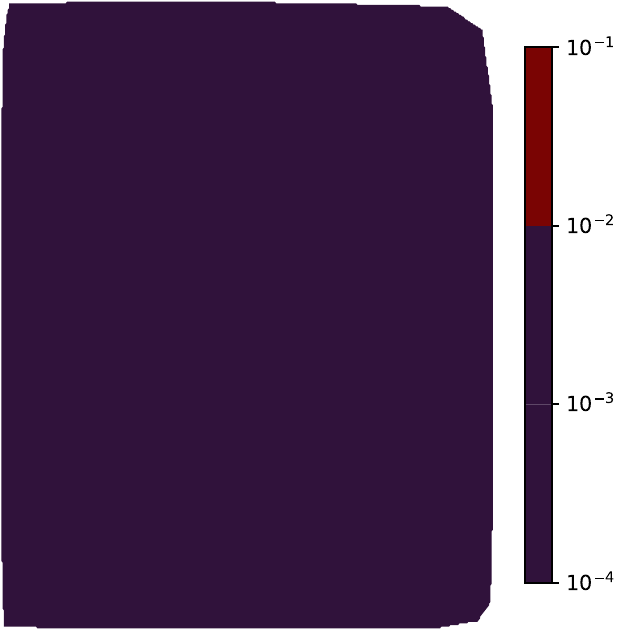}%
    \caption{Predicted by LaMO}
  \end{subfigure}

  \caption{The predicted results produced by LaMO compared to the ground truth.}\label{fig_motiv}
\end{figure}

\paragraph{3-D Brusselator problem.} We introduce a new 3-D Brusselator (diffusion-reaction equation) problem using the dataset from Laplace neural operator (LNO) \citep{cao2024laplace}. The Brusselator problem is formulated as:
\begin{align}\label{eq:reac_diff}
D\frac{\partial^2y}{\partial x^2}+ky^2- \frac{\partial y}{\partial t}=f(x,t),
\end{align}
where $y(x,t)$ represents the concentration of chemical substances or particles at location $x$ and time $t$, $f(x,t)$ is the source term and $A$ is the amplitude of the source term. In this problem, the diffusion coefficient, $D=0.01$, and the reaction rate, $k=0.01$.

\section{Distribution of selected poles reflects problem characteristics}

To understand how AMO's pole selection process is adaptive to the characteristics and nature of the problem, we illustrate the learned pole distributions for the 2-D Darcy flow problem and 3-D Brusselator problem in Figures \ref{fig:2dpole}. To clarify, here we give a brief overview of the visualization results: The distribution of selected poles for the 2-D Darcy flow problem is shown in Figures \ref{fig:2dpole} and \ref{fig:3dpole}, respectively.

\begin{figure}
    \centering
    \includegraphics[width=\linewidth]{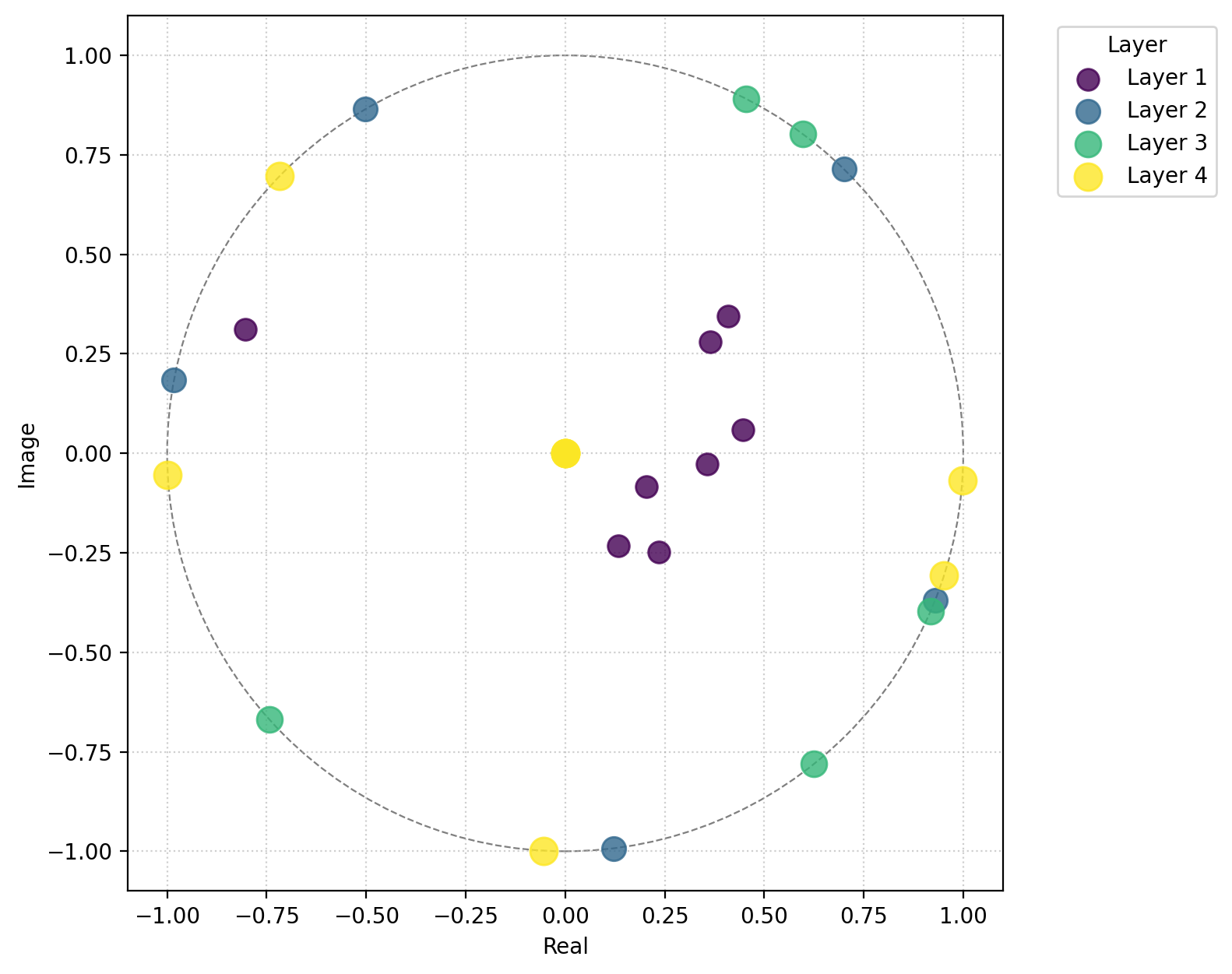}
    \caption{Learned poles distribution for the 2-D Darcy flow equation.}
    \label{fig:2dpole}
\end{figure}

\begin{figure}
    \centering
    \includegraphics[width=\linewidth]{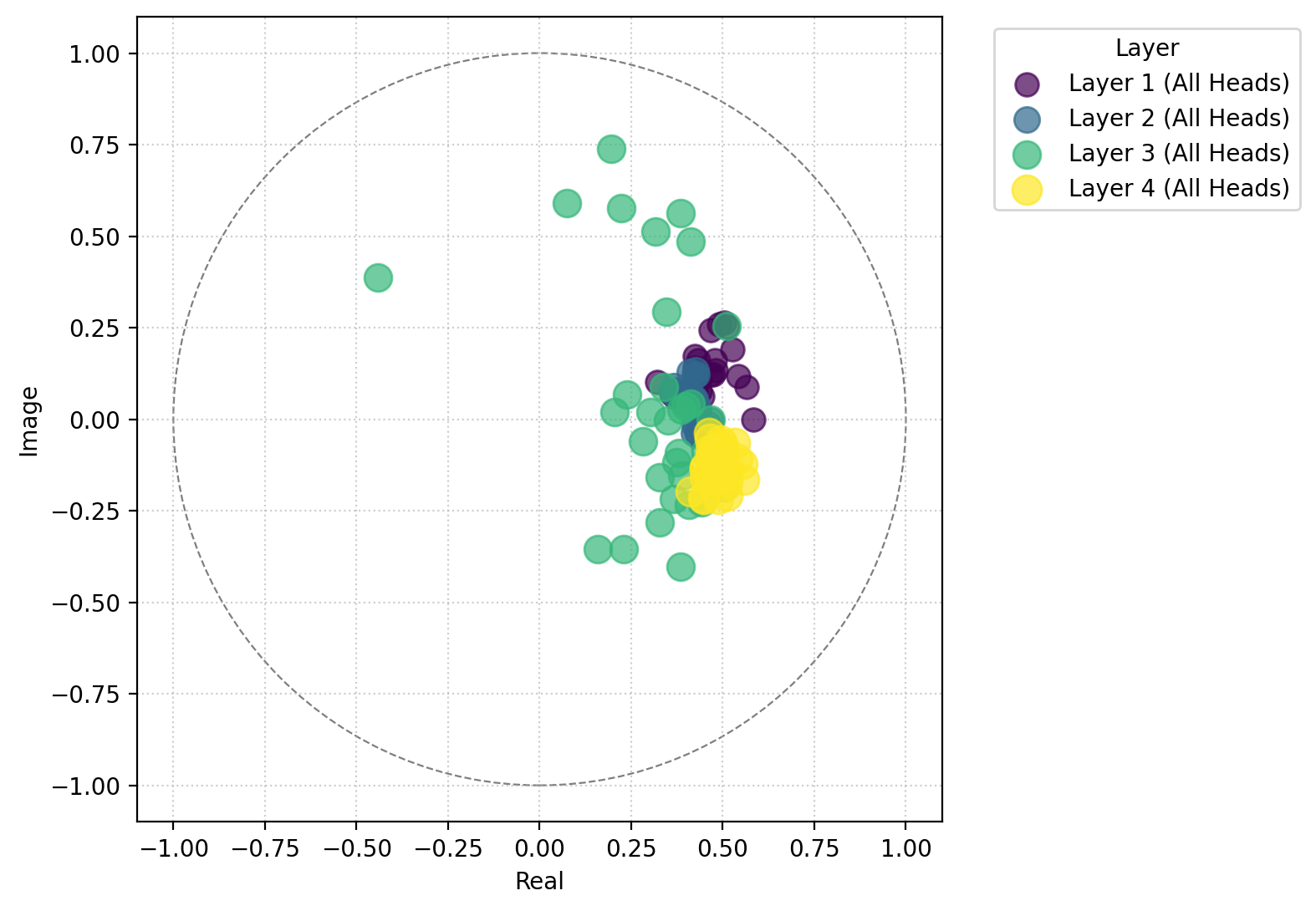}
    \caption{Learned poles distribution for the 3-D Brusselator equation.}
    \label{fig:3dpole}
\end{figure}

We observe that, across the layers, the learned poles of AMO on Darcy flow problem tend to approach to the boundary of the unit disk, while those on the Brusselator problem tend to be in the interior of the unit disk. The reason is that, Darcy flow problem is an elliptic equation, which is a smoothing operator. Thus, even though the input coefficient (the permeability) is very rough and discontinuous, the solution inside the domain will be well-behaved. Therefore, the challenging characteristics and singularities of the Darcy flow problem are located at the boundaries, and then more adaptive poles would be put there. Meanwhile, the complexity of the Brusselator problem does not come from the boundaries. It comes from the local, non-linear reaction that happens at every single point inside the domain. Therefore, most of the learned poles should be put inside the unit disk.

\section{Theoretical Results of AMO}
\label{appen_theore}

\paragraph{Basic settings.}
Let $\mathbb D=\{z\in\mathbb C:\,|z|<1\}$. Consider a reproducing kernel Hilbert space (RKHS) $(\mathcal H,\langle\cdot,\cdot\rangle_{\mathcal H})$ of complex-valued functions on $\mathbb D$ with the following properties.

\begin{assumption}\label{assump:rkhs-structure}
There is a family of normalized reproducing kernels $\{e_a: a\in\mathbb D\}\subset\mathcal H$ such that
\begin{equation}\label{eq:ea-RKHS}
e_a(z)=\frac{\sqrt{1-|a|^2}}{1-\overline a z}\in\mathcal H,
\qquad
\langle f,e_a\rangle_{\mathcal H}=f(a)\sqrt{1-|a|^2}\quad\forall\,f\in\mathcal H,\,a\in\mathbb D.
\end{equation}
Given a pole sequence $a_{1:\infty}=(a_1,a_2,\dots)\subset\mathbb D$, define the Takenaka--Malmquist (TM) system by
\begin{equation}\label{eq:TM-bases}
\mathscr B_1(z)=e_{a_1}(z),\qquad
\mathscr B_i(z)=e_{a_i}(z)\prod_{j=1}^{i-1}\frac{z-a_j}{1-\overline{a_j}z}\quad (i\ge 2).
\end{equation}
Assume $\{\mathscr B_i\}_{i\ge1}$ is an orthonormal system in $\mathcal H$, and its closed linear span equals the model space
\begin{equation}\label{eq:KB}
K_B:=\overline{\mathrm{span}}\{\mathscr B_i:\,i\ge1\}\subseteq \mathcal H,
\end{equation}
where $B$ is the Blaschke product with zeros $\{a_i\}$. 
\end{assumption}

\paragraph{AMO notation.}
Let $s\in\mathcal H$ be the latent target representation and $u^\star=\mathcal Q(s)$, where $\mathcal Q:\mathcal H\to\mathcal U$ is a Lipschitz decoder with constant $L_{\mathcal Q}$. Define the ideal TM coefficients and partial sums
\begin{equation}\label{eq:ci-star}
c_i^\star:=\langle s,\mathscr B_i\rangle_{\mathcal H},
\qquad
s_N:=\sum_{i=1}^N c_i^\star\,\mathscr B_i.
\end{equation}
AMO learns estimates $\widehat c_i$ of $c_i^\star$ (via an SSM in the frequency domain) and aggregates them through the skip connection:
\begin{equation}\label{eq:agg-update}
z_{i+1} := z_i + \widehat c_i\,\mathscr B_i,\qquad z_1:=0 .
\end{equation}

\subsection{Aggregation identity and frequency-domain coefficient extraction}

\begin{lemma}\label{lem:aggregation}
Under \ref{eq:agg-update}, one has, for every $N\in\mathbb N$,
\begin{equation}\label{eq:agg-conclusion}
z_{N+1}=\sum_{i=1}^N \widehat c_i\,\mathscr B_i.
\end{equation}
\end{lemma}

\begin{proof}
The proof is by induction. For $N=1$, $z_2=z_1+\widehat c_1\mathscr B_1=\widehat c_1\mathscr B_1$, so \ref{eq:agg-conclusion} holds. Assume \ref{eq:agg-conclusion} holds for $N$, i.e.,
$z_{N+1}=\sum_{i=1}^N \widehat c_i\,\mathscr B_i$. Then
\[
z_{N+2}=z_{N+1}+\widehat c_{N+1}\mathscr B_{N+1}
=\sum_{i=1}^{N+1}\widehat c_i\,\mathscr B_i,
\]
which establishes the claim for $N+1$.
\end{proof}

\begin{lemma}\label{lem:parseval}
Suppose the $i$-th SSM has transfer function
\begin{equation}\label{eq:Hi}
H_i(e^{i\omega})=\overline{\mathscr B_i(e^{i\omega})},
\end{equation}
so that the block multiplies the input spectrum by $\overline{\mathscr B_i}$ and outputs the zero-lag correlation. If the discrete inner product used by AMO is a consistent quadrature for $\langle\cdot,\cdot\rangle_{\mathcal H}$ on the class $\{s\}\cup\{\mathscr B_i\}$, then
\begin{equation}\label{eq:ci-equals}
\widehat c_i \to \langle s,\mathscr B_i\rangle_{\mathcal H} = c_i^\star
\quad\text{as the quadrature is refined.}
\end{equation}
\end{lemma}

\begin{proof}
By \ref{eq:Hi}, the block forms (pointwise on the grid) $Y_i=\overline{\mathscr B_i}\cdot s$ in the transform domain; the zero-lag correlation is the discretized inner product $\langle s,\mathscr B_i\rangle_{\text{disc}}$. Consistency of the quadrature implies $\langle s,\mathscr B_i\rangle_{\text{disc}}\to \langle s,\mathscr B_i\rangle_{\mathcal H}$ as the grid is refined. Hence $\widehat c_i\to c_i^\star$.
\end{proof}

\subsection{Convergence in the model space and projection error}

\begin{theorem}\label{thm:projection}
Under Assumption~\ref{assump:rkhs-structure}, if AMO recovers the exact coefficients $c_i^\star=\langle s,\mathscr B_i\rangle_{\mathcal H}$, then
\begin{equation}\label{eq:proj-conv}
s_N:=\sum_{i=1}^N c_i^\star\mathscr B_i \;\xrightarrow[N\to\infty]{\mathcal H}\; \Pi_{K_B}s,
\end{equation}
the orthogonal projection of $s$ onto $K_B$. Consequently,
\begin{equation}\label{eq:decoder-proj}
\|u^\star-\mathcal Q(s_N)\|
\;\le\;
L_{\mathcal Q}\,\|s-\Pi_{K_B}s\|_{\mathcal H} + L_{\mathcal Q}\,\|\Pi_{K_B}s-s_N\|_{\mathcal H}
\;\xrightarrow[N\to\infty]{}\; L_{\mathcal Q}\,\mathrm{dist}(s,K_B).
\end{equation}
\end{theorem}

\begin{proof}
Because $\{\mathscr B_i\}$ is an orthonormal basis (ONB) of $K_B$, the Fourier expansion of $\Pi_{K_B}s$ in this ONB has coefficients $\langle s,\mathscr B_i\rangle_{\mathcal H}$, and the $N$-th partial sum equals $s_N$. Convergence in norm to the projection is standard for orthogonal series in a Hilbert space, giving \ref{eq:proj-conv}. The bound \ref{eq:decoder-proj} follows from Lipschitz continuity of $\mathcal Q$:
\[
\|u^\star-\mathcal Q(s_N)\|
=\|\mathcal Q(s)-\mathcal Q(s_N)\|
\le L_{\mathcal Q}\|s-s_N\|
\le L_{\mathcal Q}\big(\|s-\Pi_{K_B}s\|+\|\Pi_{K_B}s-s_N\|\big).
\]
\end{proof}

\begin{remark}
No greedy or maximal selection is used. The MLP-generated poles determine $K_B$; AMO converges to $\Pi_{K_B}s$, and to $s$ whenever $s\in K_B$.
\end{remark}

\subsection{Best N-term error and rates without greedy selection}

\begin{definition}\label{def:bestN}
Let $\mathcal D:=\{\mathscr B_i(\cdot;a_{1:i}):\,a_{1:i}\in\mathbb D^i,\,i\in\mathbb N\}$ be the TM dictionary. Define the best $N$-term error
\begin{equation}\label{eq:EN}
E_N(s)\;:=\;\inf_{a_{1:N},\,c_{1:N}}
\Big\| s-\sum_{i=1}^N c_i\,\mathscr B_i(\cdot;a_{1:i})\Big\|_{\mathcal H}.
\end{equation}
\end{definition}

\begin{theorem}\label{thm:error-decomp}
Let $\tilde a_{1:N}$ be the poles output by the MLP and set $c_i^\star=\langle s,\mathscr B_i(\cdot;\tilde a_{1:i})\rangle_{\mathcal H}$. If AMO learns $\widehat c_i$, then
\begin{equation}\label{eq:error-decomp-sharp}
\Big\| s-\sum_{i=1}^N \widehat c_i\,\mathscr B_i(\cdot;\tilde a_{1:i})\Big\|_{\mathcal H}
\;\le\;
E_N(s)\;+\;\Delta_{\text{pole}}(N)\;+\;\Big(\sum_{i=1}^N |\widehat c_i-c_i^\star|^2\Big)^{\!\frac12},
\end{equation}
where
\begin{equation}\label{eq:Delta-pole}
\Delta_{\text{pole}}(N):=\inf_{c_{1:N}}\Big\| s-\sum_{i=1}^N c_i\,\mathscr B_i(\cdot;\tilde a_{1:i})\Big\|_{\mathcal H}-E_N(s)\;\ge 0.
\end{equation}
\end{theorem}

\begin{proof}
Choose $a^{\mathrm{best}}_{1:N},c^{\mathrm{best}}_{1:N}$ that attain (or $\varepsilon$-attain) $E_N(s)$ and denote
$s_N^{\mathrm{best}}:=\sum_{i=1}^N c_i^{\mathrm{best}}\mathscr B_i(\cdot;a^{\mathrm{best}}_{1:i})$.
Then
\begin{align*}
\left\| s-\sum_{i=1}^N \widehat c_i\mathscr B_i(\cdot;\tilde a_{1:i})\right\|
&\le \|s-s_N^{\mathrm{best}}\|
+\left\| s_N^{\mathrm{best}}-\sum_{i=1}^N c_i^\star\mathscr B_i(\cdot;\tilde a_{1:i})\right\|
+\left\|\sum_{i=1}^N (c_i^\star-\widehat c_i)\mathscr B_i(\cdot;\tilde a_{1:i})\right\|\\
&\le E_N(s)+\Delta_{\text{pole}}(N)
+\Big(\sum_{i=1}^N |c_i^\star-\widehat c_i|^2\Big)^{1/2}.
\end{align*}
The last inequality uses the definition of $\Delta_{\text{pole}}(N)$ and orthonormality of
$\{\mathscr B_i(\cdot;\tilde a_{1:i})\}_{i=1}^N$.
\end{proof}

\begin{corollary}\label{cor:weak-lp}
Assume for the fixed MLP-produced poles $\tilde a_{1:i}$ that the exact TM coefficients satisfy the weak-$\ell^p$ decay
\[
|c_i^\star|^\ast \le C\,i^{-1/p},\qquad 0<p<2,
\]
where $(|c_i^\star|^\ast)$ is the nonincreasing rearrangement. Then
\begin{equation}\label{eq:rate-weak-lp}
\inf_{c_{1:N}}\Big\| s-\sum_{i=1}^N c_i\,\mathscr B_i(\cdot;\tilde a_{1:i})\Big\|_{\mathcal H}
=\mathcal O\!\big(N^{\frac12-\frac1p}\big).
\end{equation}
If, in addition, $\Delta_{\text{pole}}(N)=o(1)$ and $\big(\sum_{i=1}^N|\widehat c_i-c_i^\star|^2\big)^{1/2}=o(1)$, then the AMO error in \ref{eq:error-decomp-sharp} is $\mathcal O\!\big(N^{\frac12-\frac1p}\big)$.
\end{corollary}

\begin{proof}
For an orthonormal system, the best $N$-term error equals the $\ell^2$ tail of the rearranged coefficients. With $|c_i^\star|^\ast \le C i^{-1/p}$ and $p<2$,
\[
\sum_{i>N} (|c_i^\star|^\ast)^2 \le C^2 \sum_{i>N} i^{-2/p}
=\mathcal O\!\big(N^{1-\frac{2}{p}}\big),
\]
hence the norm error (square root) is $\mathcal O(N^{\frac12-\frac1p})$.
\end{proof}

\subsection{Learning and discretization errors}

\begin{assumption}\label{assump:learn}
Each $\widehat c_i$ is obtained by ERM over $m$ i.i.d.\ frequency samples using a hypothesis class with effective capacity $d_{\mathrm{eff}}$ under sub-Gaussian noise, so that
\begin{equation}\label{eq:ERM-rate}
\mathbb E\big[|\widehat c_i-c_i^\star|\big]
=\mathcal O\!\Big(\sqrt{\tfrac{d_{\mathrm{eff}}}{m}}\Big).
\end{equation}
\end{assumption}

\begin{lemma}\label{lem:trap}
Let $\langle\cdot,\cdot\rangle_{\tilde N}$ be a discrete inner product (e.g., uniform frequency grid) that is a consistent quadrature for $\langle\cdot,\cdot\rangle_{\mathcal H}$ on the class generated by $\{s\}\cup\{\mathscr B_i\}$. Then there exists $\varepsilon_{\mathrm{disc}}(\tilde N)\downarrow 0$ such that
\begin{equation}\label{eq:disc-consistency}
\big|\langle f,g\rangle_{\mathcal H}-\langle f,g\rangle_{\tilde N}\big|
\le \varepsilon_{\mathrm{disc}}(\tilde N)\qquad
\text{for all } f\in\{s\},\, g\in\{\mathscr B_i\}_{i\ge1}.
\end{equation}
\end{lemma}

\begin{proof}
Since point evaluations are continuous linear functionals in an RKHS and the involved functions are continuous on compact subsets, standard quadrature consistency yields \ref{eq:disc-consistency}. (If $f,g$ are analytic in an annulus around the unit circle, one gets exponential rates; under Sobolev regularity, algebraic rates.)
\end{proof}

\begin{theorem}\label{thm:full-err}
Under Assumptions~\ref{assump:rkhs-structure} and \ref{assump:learn} and Lemma~\ref{lem:trap}, the AMO output after $N$ blocks and $\tilde N$ grid points satisfies
\begin{equation}\label{eq:full-err}
\|u^\star-\hat u_{N,\theta}\|
\;\le\;
L_{\mathcal Q}\Big(E_N(s)+\Delta_{\text{pole}}(N)+\Big(\sum_{i=1}^N |\widehat c_i-c_i^\star|^2\Big)^{\!1/2}\Big)
\;+\;
\varepsilon_{\mathrm{disc}}(\tilde N),
\end{equation}
with $\mathbb E[|\widehat c_i-c_i^\star|]=\mathcal O(\sqrt{d_{\mathrm{eff}}/m})$ and $\varepsilon_{\mathrm{disc}}(\tilde N)\to 0$ as $\tilde N\to\infty$.
\end{theorem}

\begin{proof}
Apply Theorem~\ref{thm:error-decomp} to bound the latent $\mathcal H$-error. Then use Lipschitz continuity of $\mathcal Q$ to transfer the bound to the output space. The discretization error adds $\varepsilon_{\mathrm{disc}}(\tilde N)$ due to \ref{eq:disc-consistency}.
\end{proof}

\subsection{Stability to pole perturbations}

\begin{lemma}\label{lem:lipschitz-factors}
For $a,b\in\mathbb D$ and $z\in\mathbb D$,
\begin{align}
\Big|\frac{1}{1-\overline a z}-\frac{1}{1-\overline b z}\Big|
&\le \frac{|a-b|}{(1-|a|)(1-|b|)},\label{eq:denom-diff}\\[2pt]
\Big|\sqrt{1-|a|^2}-\sqrt{1-|b|^2}\Big|
&\le \frac{|a-b|}{\sqrt{1-\max\{|a|,|b|\}^2}},\label{eq:sqrt-diff}
\end{align}
and for $F(z;a)=\dfrac{z-a}{1-\overline a z}$,
\begin{equation}\label{eq:F-diff}
|F(z;a)-F(z;b)|
\le \frac{4\,|a-b|}{(1-|a|)(1-|b|)},\qquad |F(z;a)|\le 1.
\end{equation}
\end{lemma}

\begin{proof}
For \ref{eq:denom-diff},
\[
\frac{1}{1-\overline a z}-\frac{1}{1-\overline b z}
=\frac{(\overline a-\overline b)z}{(1-\overline a z)(1-\overline b z)},
\]
and $|1-\overline a z|\ge 1-|a||z|\ge 1-|a|$, $|z|\le1$, yielding the bound.  
For \ref{eq:sqrt-diff}, use the mean-value theorem on $x\mapsto \sqrt{1-x}$ with $x=|a|^2,|b|^2$ and $||a|^2-|b|^2|\le |a-b|(|a|+|b|)\le 2|a-b|$.  
For \ref{eq:F-diff}, expand
\[
F(z;a)-F(z;b)=\frac{(b-a)+(\overline a-\overline b)z^2+(a\overline b-b\overline a)z}{(1-\overline a z)(1-\overline b z)},
\]
and bound the numerator by $C|a-b|$ for $|z|\le1$, while the denominator is bounded below by $(1-|a|)(1-|b|)$.
\end{proof}

\begin{theorem}\label{thm:stability-poles}
Let $a_{1:i},\tilde a_{1:i}\in\mathbb D$ with $|\tilde a_j-a_j|\le \delta_j$. Then there exist constants $C_i>0$ (depending on $a_{1:i}$) such that
\begin{equation}\label{eq:Bi-stability}
\|\mathscr B_i(\cdot;\tilde a_{1:i})-\mathscr B_i(\cdot;a_{1:i})\|_{\mathcal H}
\;\le\;
C_i\,\sum_{j=1}^i \frac{\delta_j}{1-|a_j|}.
\end{equation}
Consequently, for any coefficients $\widehat c_i$,
\begin{equation}\label{eq:sum-stability}
\Big\|\sum_{i=1}^N \widehat c_i\,\mathscr B_i(\cdot;\tilde a_{1:i})-\sum_{i=1}^N \widehat c_i\,\mathscr B_i(\cdot;a_{1:i})\Big\|_{\mathcal H}
\;\le\;
\Big(\sum_{i=1}^N |\widehat c_i|\,C_i\Big)\,\Big(\sum_{j=1}^N \frac{\delta_j}{1-|a_j|}\Big).
\end{equation}
\end{theorem}

\begin{proof}
Write
\[
\mathscr B_i(\cdot;a_{1:i})=e_{a_i}\prod_{j=1}^{i-1}F(\cdot;a_j),
\qquad
\mathscr B_i(\cdot;\tilde a_{1:i})=e_{\tilde a_i}\prod_{j=1}^{i-1}F(\cdot;\tilde a_j).
\]
Use the product telescoping identity
\[
\prod_{k=1}^{i}P_k-\prod_{k=1}^{i}Q_k
=\sum_{k=1}^{i}\Big(\prod_{j<k}P_j\Big)(P_k-Q_k)\Big(\prod_{j>k}Q_j\Big),
\]
with $P_1=e_{\tilde a_i}$, $Q_1=e_{a_i}$, and $P_{k}=F(\cdot;\tilde a_{k-1})$, $Q_{k}=F(\cdot;a_{k-1})$ for $k\ge2$. Taking sup-norms on $\mathbb D$ and using $|F(\cdot;a)|\le 1$,
\[
\|\mathscr B_i(\cdot;\tilde a_{1:i})-\mathscr B_i(\cdot;a_{1:i})\|_\infty
\;\le\;
\|e_{\tilde a_i}-e_{a_i}\|_\infty
+\sum_{j=1}^{i-1}\|F(\cdot;\tilde a_j)-F(\cdot;a_j)\|_\infty.
\]
Apply Lemma~\ref{lem:lipschitz-factors} to bound each term by a constant times $\delta_j/(1-|a_j|)$. Since evaluation functionals are continuous and the kernel is bounded on compact subsets, there exists an embedding constant $C_{\mathrm{emb}}$ with
$\|f\|_{\mathcal H}\le C_{\mathrm{emb}}\|f\|_\infty$ on the set considered; thus \ref{eq:Bi-stability} follows with $C_i$ absorbing all constants. Finally,
\[
\Big\|\sum_{i=1}^N \widehat c_i\big(\mathscr B_i(\cdot;\tilde a_{1:i})-\mathscr B_i(\cdot;a_{1:i})\big)\Big\|_{\mathcal H}
\le \sum_{i=1}^N |\widehat c_i|\,
\|\mathscr B_i(\cdot;\tilde a_{1:i})-\mathscr B_i(\cdot;a_{1:i})\|_{\mathcal H},
\]
giving \ref{eq:sum-stability}.
\end{proof}

\subsection{End-to-end convergence without greedy selection}

\begin{theorem}\label{thm:final-conv}
Assume:
\begin{enumerate}\itemsep2pt
\item $s\in K_B$;
\item $\sum_{i=1}^\infty \mathbb E[|\widehat c_i-c_i^\star|^2]^{1/2}<\infty$ (as sample size $m\to\infty$ and model capacity increase);
\item $\varepsilon_{\mathrm{disc}}(\tilde N)\to 0$ as $\tilde N\to\infty$.
\end{enumerate}
Then
\[
\lim_{N\to\infty}\|u^\star-\hat u_{N,\theta}\|=0.
\]
\end{theorem}

\begin{proof}
Since $s\in K_B$ and $\{\mathscr B_i\}$ is an ONB of $K_B$, Theorem~\ref{thm:projection} gives $s_N\to s$ in $\mathcal H$. In \ref{eq:full-err}, for this fixed pole sequence one has $E_N(s)=\Delta_{\text{pole}}(N)=0$. Using (2) and (3), we obtain $\|u^\star-\hat u_{N,\theta}\|\to 0$.
\end{proof}

\subsection{Connection of SSM to correlation and AMO output}

\begin{proposition}\label{prop:ssm-equals-corr}
With $H_i(e^{i\omega})=\overline{\mathscr B_i(e^{i\omega})}$, the $i$-th SSM block computes $\widehat c_i\approx\langle z_i,\mathscr B_i\rangle_{\mathcal H}$. Hence, by Lemma~\ref{lem:aggregation}, after $N$ blocks
\begin{equation}\label{eq:znplus}
z_{N+1}=\sum_{i=1}^N \widehat c_i\,\mathscr B_i,
\qquad
\hat u_{N,\theta}=\mathcal Q(z_{N+1}).
\end{equation}
\end{proposition}

\begin{proof}
The coefficient claim follows from Lemma~\ref{lem:parseval} applied to $z_i$ in place of $s$. The aggregation identity is Lemma~\ref{lem:aggregation}. The last equality is the definition of $\mathcal{Q}$.
\end{proof}

\begin{corollary}\label{cor:output-space}
All latent-space error bounds transfer to the PDE output space via
\[
\|u^\star-\hat u_{N,\theta}\|\le L_{\mathcal Q}\,\Big\|s-\sum_{i=1}^N \widehat c_i \mathscr B_i\Big\|
+\varepsilon_{\mathrm{disc}}(\tilde N).
\]
\end{corollary}
\section{Transfer function}\label{appen_trans}
We consider a (finite) Blaschke product
\begin{equation}
  H(z)
  \;=\;
  \prod_{j=1}^{n}\frac{1 - p_j z}{\,z - p_j\,},
  \qquad |p_j|<1,
  \label{eq:blaschke}
\end{equation}
and convert it into a single ratio of polynomials whose coefficients match 
the parameterization used to train SSMs.

\paragraph{Polynomial expansion and $z^{-1}$ form.}
Denote numerator and denominator polynomials
\begin{equation}
  B_{\mathrm{poly}}(z)=\prod_{j=1}^{n}(z-p_j),
  \qquad
  A_{\mathrm{poly}}(z)=\prod_{j=1}^{n}(1-p_j z),
  \label{eq:ABpoly}
\end{equation}
so that $H(z)=\frac{A_{\mathrm{poly}}(z)}{B_{\mathrm{poly}}(z)}$.
Let $d=\deg B_{\mathrm{poly}}=\deg A_{\mathrm{poly}}=n$.
To obtain the form with a unit constant term in the denominator,
divide numerator and denominator by $z^{d}$ and then normalize:
\begin{equation}
  \widetilde{H}(z)
  \;=\;
  \frac{\sum_{k=0}^{d}\,\alpha_k z^{-k}}
       {\sum_{k=0}^{d}\,\beta_k z^{-k}}
  \;\; \xrightarrow{\text{normalize}}\;\;
  h_0 \;+\; \sum_{k=1}^{d} \frac{b_k}{1}\,z^{-k}
  \;\Big/ \;
  \Big(1 \;+\; \sum_{k=1}^{d} a_k z^{-k}\Big).
  \label{eq:zminus1}
\end{equation}
The SSM coefficients are then reduced as:
\[
  h_0 \!=\! \frac{\alpha_0}{\beta_0},
  \qquad
  b_k \!=\! \frac{\alpha_k}{\beta_0},
  \qquad
  a_k \!=\! \frac{\beta_k}{\beta_0},
  \quad k=1,\dots,d.
\]

\paragraph{Example ($n=2$).}
With $p_1,p_2\in\mathbb{C}$, expand
\[
  B_{\mathrm{poly}}(z)=(z-p_1)(z-p_2)
  = z^2 - (p_1{+}p_2)z + p_1p_2,
\]
\[
  A_{\mathrm{poly}}(z)=(1-p_1 z)(1-p_2 z)
  = 1 - (p_1{+}p_2)z + (p_1p_2)z^2.
\]
Divide by $z^2$ to get polynomials in $z^{-1}$ and normalize by the denominator's constant term
($\beta_0=p_1p_2$), yielding
\[
  H(z)
  \;=\;
  \frac{1 - (p_1{+}p_2)z^{-1} + (p_1p_2) z^{-2}}
       {\,p_1p_2 - (p_1{+}p_2)z^{-1} + z^{-2}\,}
  \;=\;
  \frac{h_0 + b_1 z^{-1} + b_2 z^{-2}}{1 + a_1 z^{-1} + a_2 z^{-2}},
\]
with
\[
  h_0=\frac{1}{p_1p_2},\quad
  b_1=-\frac{p_1+p_2}{p_1p_2},\quad
  b_2=1,\qquad
  a_1=-\frac{p_1+p_2}{p_1p_2},\quad
  a_2=\frac{1}{p_1p_2}.
\]

\paragraph{Efficient computation for large \texorpdfstring{$n$}{n}.} Direct symbolic expansion scales poorly. Instead, we multiply degree-1 polynomials using FFT-based
convolution. Represent each factor by its coefficient vector:
\[
  (z-p_j) \;\leftrightarrow\; [1,\,-p_j],\qquad
  (1-p_j z) \;\leftrightarrow\; [1,\,-p_j],
\]
and iteratively convolve to form $B_{\mathrm{poly}}$ and $A_{\mathrm{poly}}$.
By the convolution theorem, polynomial multiplication is element-wise in the frequency domain,
giving $\mathcal{O}(d\log d)$ complexity. After both polynomials are assembled, convert to $z^{-1}$ by dividing by $z^d$, then normalize by the denominator's constant term to obtain $(h_0,\{a_k\},\{b_k\})$ as in~\ref{eq:zminus1}.
\end{document}